\documentclass[3p,onecolumn]{elsarticle}

\usepackage{amsmath,amsfonts,amssymb,amscd}
\usepackage{graphics}
\usepackage{epsfig}
\usepackage{mathrsfs}
\usepackage{latexsym}

\usepackage{amsthm}

\newtheorem{theorem}{Theorem}[section]
\newtheorem{corollary}[theorem]{Corollary} 
\newtheorem{lemma}[theorem]{Lemma}

\theoremstyle{definition}

\theoremstyle{remark}
\newtheorem{remark}[theorem]{Remark}

\newtheorem{example}[theorem]{Example}
\numberwithin{equation}{section}
\newtheorem{definition}[theorem]{Definition}

\newcommand{\norm}[1]{\Vert#1\Vert}
\newcommand{\e}{{\epsilon}}
\newcommand{\E}{{\mathbb E}}
\newcommand{\VC}{{\mbox{VC}}}
\newcommand{\fat}{{\mbox{fat}}}
\newcommand{\abs}[1]{\lvert#1\rvert}
\def\fat{{\mathrm{fat}}}
\def\Q {{\mathbb Q}}

\def\C {{\mathbb C}}

\def\N{{\mathbb N}}

\def\R{{\mathbb R}}
\def\modd{\,{\mathrm{mod}}\,}

\newcommand{\neswarrow}{\mathrel{\text{$\nearrow$\llap{$\swarrow$}}}}

\journal{Theoretical Computer Science}

\begin{document}

\begin{frontmatter}



\title{PAC learnability under non-atomic measures: a problem by Vidyasagar}


\author{Vladimir Pestov}

\address{Departamento de Matem\'atica,
Universidade Federal de Santa Catarina,
Campus Universit\'ario Trindade,
CEP 88.040-900 Florian\'opolis-SC, Brasil \footnote{Pesquisador Visitante do CNPq.}} 

\address{Department of Mathematics and Statistics, University of Ottawa,
        585 King Edward Avenue, Ottawa, Ontario, K1N6N5 Canada \footnote{Permanent address.}}

\begin{abstract} 
In response to a 1997 problem of M. Vidyasagar, we state a criterion for PAC learnability of a concept class $\mathscr C$ under the family of all non-atomic (diffuse) measures on the domain $\Omega$. 
The uniform Glivenko--Cantelli property with respect to non-atomic measures is no longer a necessary condition, and consistent learnability cannot in general be expected.
Our criterion is stated in terms of a combinatorial parameter $\VC({\mathscr C}\,{\mathrm{mod}}\,\omega_1)$ which we call the VC dimension of $\mathscr C$ modulo countable sets. The new parameter is obtained by ``thickening up'' single points in the definition of VC dimension to uncountable ``clusters''. Equivalently, $\VC(\mathscr C\modd\omega_1)\leq d$ if and only if every countable subclass of $\mathscr C$ has VC dimension $\leq d$ outside a countable subset of $\Omega$. The new parameter can be also expressed as the classical VC dimension of $\mathscr C$ calculated on a suitable subset of a compactification of $\Omega$. We do not make any measurability assumptions on $\mathscr C$, assuming instead the validity of Martin's Axiom (MA). Similar results are obtained for function learning in terms of fat-shattering dimension modulo countable sets, but, just like in the classical distribution-free case, the finiteness of this parameter is sufficient but not necessary for PAC learnability under non-atomic measures.
\end{abstract}


\end{frontmatter}


\section{Introduction}
A fundamental result of statistical learning theory says that under some mild measurability assumptions on a concept class $\mathscr C$ the three conditions are equivalent: (1) $\mathscr C$ is distribution-free PAC learnable over the family $P(\Omega)$ of all probability measures on the domain $\Omega$, (2) $\mathscr C$ is a uniform Glivenko--Cantelli class with respect to $P(\Omega)$, and (3) the Vapnik--Chervonenkis dimension of $\mathscr C$ is finite \cite{VC1968,VC1971,BEHW}.
In this paper we are interested in the problem, discussed by Vidyasagar in both editions of his book \cite{vidyasagar1997,vidyasagar2003} as problem 12.8, of giving a similar combinatorial description of concept classes $\mathscr C$ which are PAC learnable under the family $P_{na}(\Omega)$ of all non-atomic probability measures on $\Omega$. (A measure $\mu$ is {\em non-atomic}, or {\em diffuse,} if every set $A$ of strictly positive measure contains a subset $B$ with $0<\mu(B)<\mu(A)$.)

The condition $\VC({\mathscr C})<\infty$, while of course sufficient for $\mathscr C$ to be learnable under $P_{na}(\Omega)$,
is not necessary. 
Let a concept class $\mathscr C$ consist of all finite and all cofinite subsets of a standard Borel space $\Omega$. Then $\VC({\mathscr C})=\infty$, and moreover $\mathscr C$ is clearly not a uniform Glivenko-Cantelli class {\em with respect to non-atomic measures.} At the same time, $\mathscr C$ is PAC learnable under non-atomic measures: any learning rule $\mathcal L$ consistent with the subclass $\{\emptyset,\Omega\}$ will learn $\mathscr C$. Notice that $\mathscr C$ is not {\em consistently} learnable under non-atomic measures: there are consistent learning rules mapping every training sample to a finite set, and they will not learn any cofinite subset of $\Omega$. 

The most salient feature of this example is that PAC learnability of a concept class $\mathscr C$ under non-atomic measures is not affected by adding to $\mathscr C$ symmetric differences $C\bigtriangleup N$ for each $C\in {\mathscr C}$ and every 
countable set $N$.

A version of VC dimension oblivious to this kind of set-theoretic ``noise'' is obtained from the classical definition by ``thickening up'' individual points and replacing them with uncountable clusters (Figure \ref{fig:shattered-3}).

\begin{figure}[ht]
\begin{center}
\scalebox{0.225}[0.225]{\includegraphics{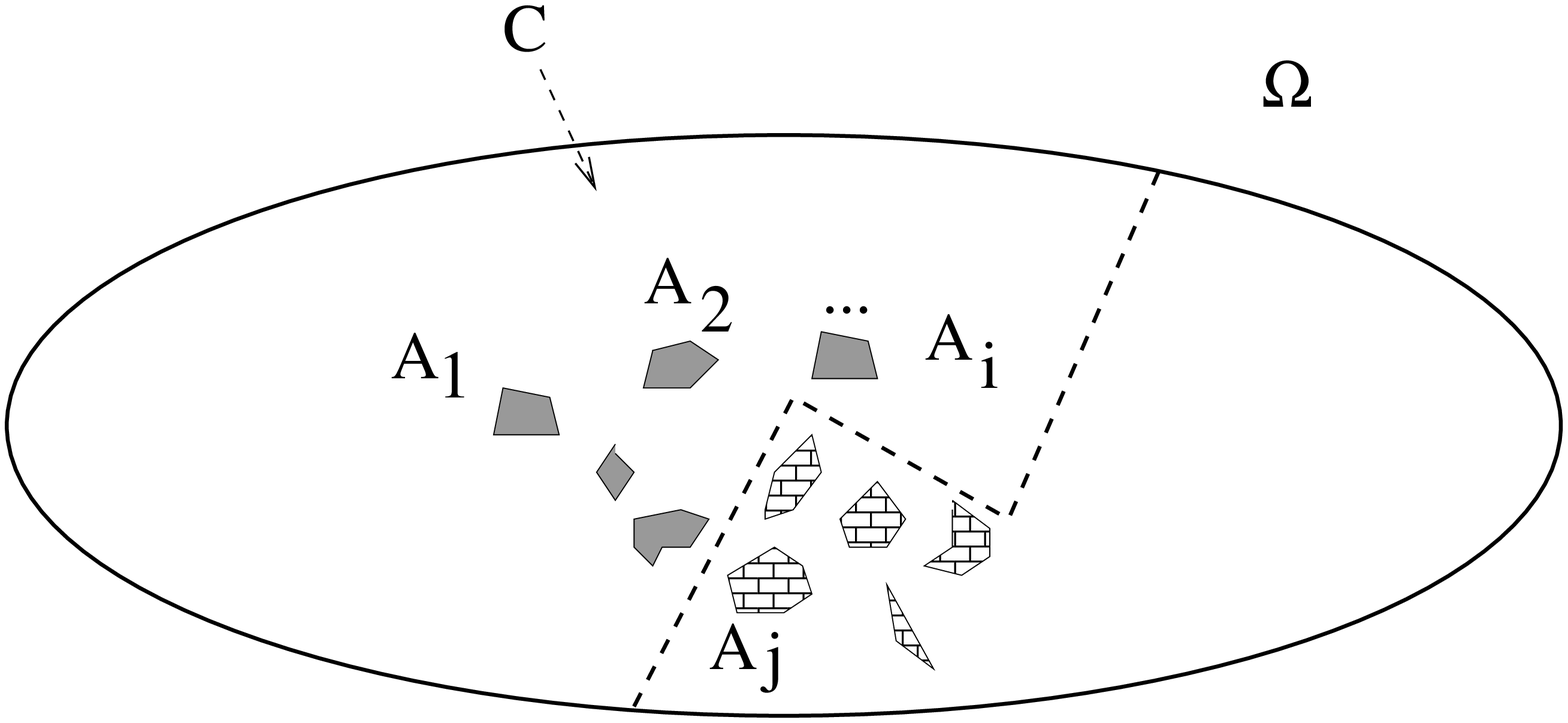}} 
\caption{A family $A_1,A_2,\ldots,A_n$ of uncountable sets shattered by $\mathscr C$.}
\label{fig:shattered-3}
\end{center}
\end{figure}

Define the {\em VC dimension of a concept class $\mathscr C$ modulo countable sets} as the supremum of natural $n$ for which there exists a family of $n$ uncountable sets, $A_1,A_2,\ldots,A_n\subseteq\Omega$, shattered by $\mathscr C$ in the sense that for each $J\subseteq \{1,2,\ldots,n\}$, there is $C\in {\mathscr C}$ 
which contains all sets $A_i$, $i\in J$, and is disjoint from all sets $A_j$, $j\notin J$. Denote this parameter by $\VC({\mathscr C}\modd\omega_1)$. Clearly, for every concept class $\mathscr C$ 
\[\VC({\mathscr C}\modd\omega_1)\leq \VC({\mathscr C}).\]
In our example above, one has $\VC({\mathscr C}\modd\omega_1)=1$, even as $\VC({\mathscr C})=\infty$.

Our main theorem for PAC concept learning under non-atomic measures requires an additional set-theoretic hypothesis, {\em Martin's Axiom} ({\em MA}) \cite{fremlin,jech,kunen}. This is one of the most often used and best studied additional set-theoretic assumptions beyond the standard Zermelo-Frenkel set theory with the Axiom of Choice (ZFC). Here is one of the equivalent forms. Let $B$ be a Boolean algebra satisfying the countable chain condition (that is, every family of pairwise disjoint elements of $B$ is countable). Then for every family ${\mathcal X}$ of cardinality $< 2^{\aleph_0}$ of subsets of $B$ there is a maximal ideal $\xi$ (element of the Stone space of $B$) with the property: each $X\in {\mathcal X}$ disjoint from $\xi$ admits an upper bound $x\notin\xi$. 

The above conclusion holds unconditionally if $\mathcal X$ is countable (due to the Baire Category Theorem), and thus Martin's Axiom follows from the Continuum Hypothesis (CH). At the same time, MA is compatible with the negation of CH, and in fact it is namely the combination MA+$\neg$CH that is really interesting. As a consequence of Martin's Axiom, the usual sigma-additivity of a measure can be strengthened as follows: the union of $<2^{\aleph_0}$ Lebesgue measurable sets is Lebesgue measurable. Essentially, this is the only property we need in the proof of the following result.

\begin{theorem}
\label{th:main}
Let $(\Omega,{\mathscr A})$ be a standard Borel space, and let ${\mathscr C}\subseteq {\mathscr A}$ be a concept class. Under Martin's Axiom, the following are equivalent.

\begin{enumerate}
\item \label{mainth:1} 
$\mathscr C$ is PAC learnable under the family of all non-atomic measures.
\item \label{mainth:3} 
$\VC({\mathscr C}\modd\omega_1)=d<\infty$. 
\item \label{mainth:4} 
Every countable subclass ${\mathscr C}^\prime\subseteq {\mathscr C}$ has finite VC dimension on the complement to some countable subset of $\Omega$ (which depends on ${\mathscr C}^\prime$).
\item\label{mainth:4a} There is $d$ such that for every countable ${\mathscr C}^\prime\subseteq {\mathscr C}$ one has $\VC({\mathscr C}^\prime)\leq d$ on the complement to some countable subset of $\Omega$ (depending on $\mathscr C^\prime$). 
\item\label{mainth:5a} Every countable subclass ${\mathscr C}^\prime\subseteq {\mathscr C}$ is a uniform Glivenko--Cantelli class with respect to the family of non-atomic measures.
\item\label{mainth:5} Every countable subclass ${\mathscr C}^\prime\subseteq {\mathscr C}$ is a uniform Glivenko--Cantelli class with respect to the family of non-atomic measures, with sample complexity $s(\e,\delta)$ which only depends on $\mathscr C$ and not on ${\mathscr C}^\prime$.
\end{enumerate}
If $\mathscr C$ is universally separable \cite{pollard}, the above are also equivalent to:
\begin{enumerate}
\item[7.] $\VC$ dimension of $\mathscr C$ is finite outside of a countable subset of $\Omega$.
\item[8.] $\mathscr C$ is a uniform Glivenko-Cantelli class with respect to the family of non-atomic probability measures.
\item[9.]
$\mathscr C$ is consistently PAC learnable under the family of all non-atomic measures.
\end{enumerate}
\end{theorem}

%

Notice that for universally separable classes, (\ref{mainth:1})--(9) are pairwise equivalent without additional set-theoretic assumptions. (A class $\mathscr C$ is {\em universally separable} if it contains a countable subclass $\mathscr C^\prime$ which is {\em universally dense}: for each $C\in {\mathscr C}$ there is a sequence $(C_n)$, $C_n\in {\mathscr C}^\prime$, such that the indicator functions $I_{C_n}$ converge to $I_C$ pointwise.)
The concept class in the above example (which is even image admissible Souslin \cite{dudley}, but not universally separable) shows that in general (7), (8) and (9) are not equivalent to the remaining conditions. 

The core of Theorem \ref{th:main} --- and the main technical novelty of our paper --- is the proof of the implication (\ref{mainth:4})$\Rightarrow$(\ref{mainth:1}). It is based on a special choice of a consistent learning rule $\mathcal L$ having the property that for every concept $C\in{\mathscr C}$, the image of all learning samples of the form $(\sigma,C\cap\sigma)$ under $\mathcal L$ forms a uniform Glivenko--Cantelli class. It is for establishing this property of $\mathcal L$ that we need Martin's Axiom.

Most of the remaining implications are relatively straightforward adaptations of the standard techniques of statistical learning. Nevertheless, 
(\ref{mainth:3})$\Rightarrow$(\ref{mainth:4}) requires a certain technical dexterity, and we study this implication in the setting of Boolean algebras.

An analog of Theorem \ref{th:main} also holds for PAC learning of function classes. In this case, we are employing a version of fat shattering dimension \cite{ABDCBH}, which we call fat shattering dimension modulo countable sets and denote $\fat_{\e}({\mathscr F}\modd\omega_1)$.
However, just like in the classical case, finiteness of this combinatorial parameter at every scale $\e>0$, while sufficient for PAC learnability of a function class $\mathscr F$ under non-atomic measures, is not necessary. It is easy to construct a function class $\mathscr F$ with $\fat_{\e}({\mathscr F}\modd\omega_1)=\infty$ which is distribution-free probably {\em exactly} learnable (Example \ref{ex:fprecpacinfdim}).

Recall that a function $f\colon X\to Y$ between two measurable spaces (sets equipped with sigma-algebras of subsets) is {\em universally measurable} if for every measurable subset $A\subseteq Y$ and every probability measure $\mu$ on $X$ the set $f^{-1}(A)$ is $\mu$-measurable. For instance, Borel functions are universally measurable.

\begin{theorem}
\label{th:fmain}
Let $\Omega$ be a standard Borel space, and let ${\mathscr F}$ be a class of universally measurable functions on $\Omega$ with values in $[0,1]$. Consider the following conditions.

\begin{enumerate}
\item \label{fmainth:1} 
$\mathscr F$ is PAC learnable under the family of all non-atomic measures.
\item \label{fmainth:3} 
For every $\e>0$, $\fat_{\e}({\mathscr F}\modd\omega_1)=d(\e)<\infty$. 
\item \label{fmainth:4} 
For each $\e>0$, every countable subclass ${\mathscr F}^\prime\subseteq {\mathscr F}$ has finite $\e$-fat shattering dimension on the complement to some countable subset of $\Omega$ (which depends on ${\mathscr F}^\prime$).
\item\label{fmainth:4a} There is a function $d(\e)$ such that for every countable ${\mathscr F}^\prime\subseteq {\mathscr F}$ and all $\e>0$ one has $\fat_{\e}({\mathscr F}^\prime)\leq d(\e)$ on the complement to some countable subset of $\Omega$ (depending on $\mathscr F^\prime$). 
\item\label{fmainth:5a} Every countable subclass ${\mathscr F}^\prime\subseteq {\mathscr F}$ is a uniform Glivenko--Cantelli class with respect to the family of non-atomic measures.
\item\label{fmainth:5} Every countable subclass ${\mathscr F}^\prime\subseteq {\mathscr F}$ is a uniform Glivenko--Cantelli class with respect to the family of non-atomic measures, with sample complexity $s(\e,\delta)$ which only depends on $\mathscr F$ and not on ${\mathscr F}^\prime$.
\end{enumerate}
The conditions (\ref{fmainth:3})--(\ref{fmainth:5}) are pairwise equivalent, and under Martin's Axiom each of them implies (\ref{fmainth:1}).
If $\mathscr F$ is universally separable, the conditions (\ref{fmainth:3})--(\ref{fmainth:5}) are also equivalent to:
\begin{enumerate}
\item[7.] For each $\e>0$, $\e$-fat shattering dimension of $\mathscr F$ is finite outside of a countable subset of $\Omega$.
\item[8.] $\mathscr F$ is a uniform Glivenko-Cantelli class with respect to the family of non-atomic probability measures,
\end{enumerate}
and each of them implies
\begin{enumerate}
\item[9.]
$\mathscr F$ is consistently PAC learnable under the family of all non-atomic measures.
\end{enumerate}
\end{theorem}

We begin the paper by reviewing a general formal setting for PAC learnability, after which we proceed to analysis of a well-known example of a concept class of VC dimension $1$ which is not a uniform Glivenko--Cantelli class and is not consistently PAC learnable \cite{DD,BEHW}. The example was originally constructed under the Continuum Hypothesis, though in fact Martin's Axiom suffices. We observe that the class $\mathscr C$ in the example is still PAC learnable, and this observation provides a clue to our approach to constructing learning rules. 

This analysis is followed by a series of general results about PAC learnability of a function class $\mathscr F$ under non-atomic measures under Martin's Axiom and without making any assumptions on measurability of $\mathscr F$ except the measurability of individual members $f$ of the class.

In the two sections to follow, we discuss Boolean algebras which appear to provide a useful framework for studying concept learning under intermediate families of measures, and commutative $C^\ast$-algebras and their spaces of maximal ideals, which provide a similar convenient framework for function classes. 
In particular, we will show that for a concept class $\mathscr C$ our version of the VC dimension modulo countable sets, $\VC({\mathscr C}\modd\omega_1)$, is just the usual VC dimension of the family of closures, ${\mathrm{cl}}(C)$, of all $C\in {\mathscr C}$, taken in a suitable compactification $b\,\Omega$ of $\Omega$ and computed over a certain subdomain of $b\,\Omega$, as illustrated in Figure \ref{fig:stone_space3}.

\begin{figure}[ht]
\begin{center}
  \scalebox{0.275}[0.275]{\includegraphics{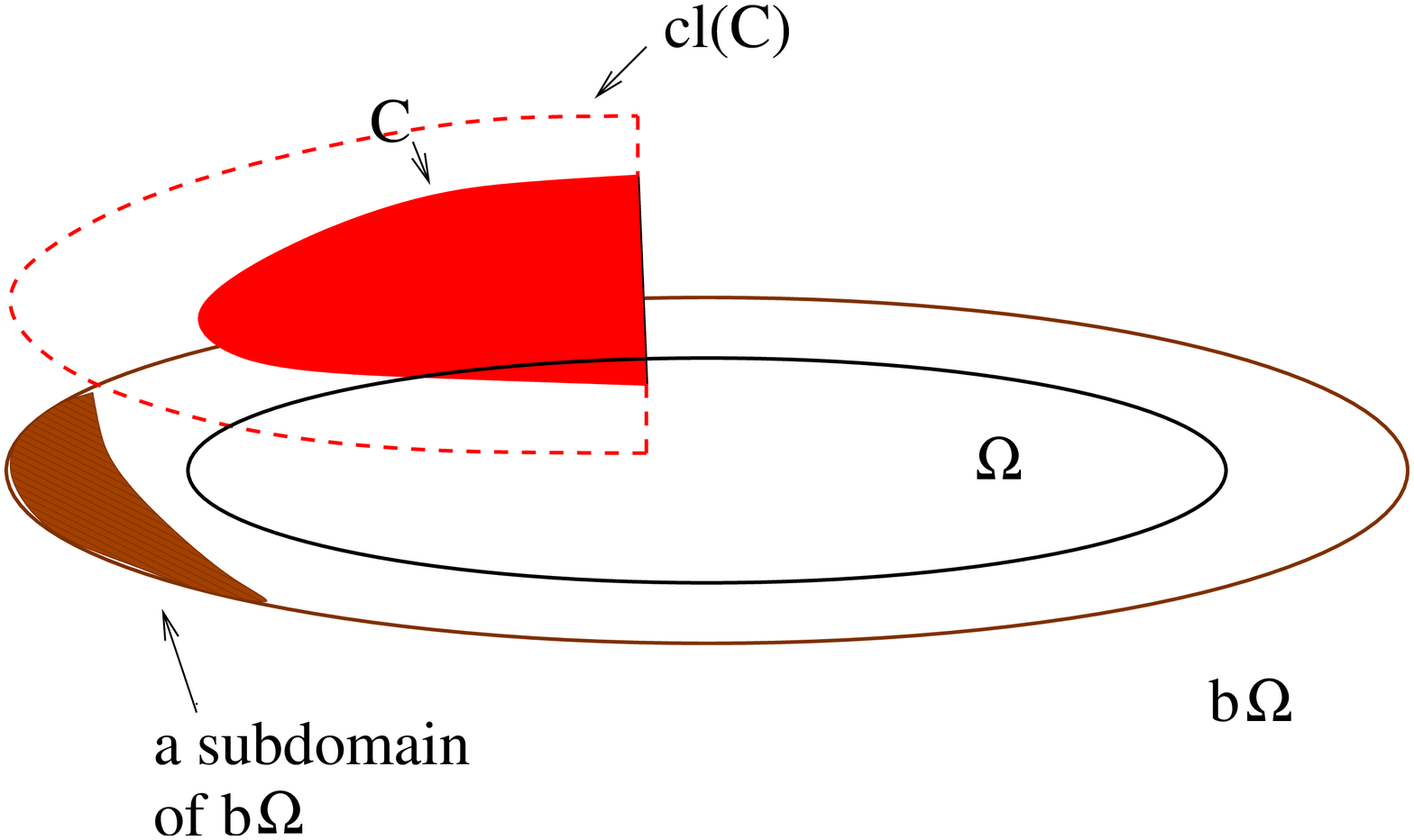}} 
  \caption{$\VC({\mathscr C}\modd\omega_1)$ via the usual VC dimension of $\mathscr C$.}
\label{fig:stone_space3}
\end{center}
\end{figure}

A similar result holds for the fat shattering dimension.

At the next stage we establish the corresponding parts of Theorems \ref{th:main} and \ref{th:fmain} for universally separable classes, at which moment we have all the machinery needed to accomplish the general case.

A conference version of this paper \cite{pestov:2010a} treated the case of concept classes, but we believe that the presentation of our approach has now improved considerably.

\section{The setting}

We need to fix a precise setting, which is mostly standard \cite{vidyasagar1997,vidyasagar2003}, see also \cite{ABDCBH,BEHW,mendelson:03}.
The {\em domain} ({\em instance space}) $\Omega=(\Omega,{\mathscr A})$ is a {\em measurable space,} that is, a set $\Omega$ equipped with a sigma-algebra of subsets $\mathscr A$. Typically, $\Omega$ is assumed to be a {\em standard Borel space,} that is, a complete separable metric space equipped with the sigma-algebra of Borel subsets. We will clarify the assumption whenever necessary. 

In the learning model, a set $\mathcal P$ of probability measures on $\Omega$ is fixed. Usually either ${\mathcal P}=P(\Omega)$ is the set of all probability measures (distribution-free learning), or ${\mathcal P}=\{\mu\}$ is a single measure (learning under fixed distribution). In our article, the case of interest is the family ${\mathcal P}=P_{na}(\Omega)$ of all non-atomic measures.

We will not distinguish between a measure $\mu$ and its Lebesgue completion, that is, an extension of $\mu$ over the larger sigma-algebra of Lebesgue measurable subsets of $\Omega$. Consequently, we will sometimes use the term {\em measurability} meaning {\em Lebesgue measurability}. No confusion can arise here.

A {\em function class}, $\mathscr F$, is a family of functions from $\Omega$ to the unit interval $[0,1]$ which are measurable with regard to every $\mu\in {\mathcal P}$. For instance, elements of $\mathscr F$ can be universally measurable, or most often Borel. A {\em concept class}, $\mathscr C$, is a function class with values in $\{0,1\}$ or, equivalently, a family of measurable subsets of $\Omega$. 

Every probability measure $\mu$ on $\Omega$ determines an $L^1$ distance between functions: 
\[\norm{f-g}_1=\int_{\Omega}\abs{f(x)-g(x)}d\mu(x).\]
For concept classes, this reduces to the following metric:
\[d_{\mu}(A,B)=\mu\left(A\bigtriangleup B \right).\] 

Often it is convenient to approximate the functions from $\mathscr F$ with elements of the {\em hypothesis space,} $\mathscr H$, which is, technically, a family of functions whose closure in each space $L^1(\mu)$, $\mu\in{\mathcal P}$, contains $\mathscr F$. However, in our article we make no distinction between $\mathscr H$ and $\mathscr F$.

A {\em learning sample} is a pair $(\sigma,r)$, where $\sigma$ is a finite subset of $\Omega$ and $r$ is a function from $\sigma$ to $[0,1]$. It is convenient to assume that elements $x_1,x_2,\ldots,x_n\in\sigma$ are ordered, and thus the set of all samples $(\sigma,r)$ with $\abs\sigma=n$ can be identified with $\left(\Omega\times [0,1]\right)^n$. In the case of concept classes, a learning sample is simply a pair $(\sigma,\tau)$ of finite subsets of $\Omega$, where $\tau\subseteq\sigma$ is thought of as the set of points where $r$ takes the value $1$. The set of all samples of size $n$ in this case is $\left(\Omega\times \{0,1\}\right)^n$.

A {\em learning rule} (for $\mathscr F$) is a mapping 
\[{\mathcal L}\colon \bigcup_{n=1}^\infty\Omega^n\times [0,1]^n\to {\mathscr F}\]
which satisfies the following measurability condition:  for every $f\in{\mathscr F}$ and $\mu\in{\mathcal P}$, the function
\begin{equation}
\label{eq:measurabilityl}
\Omega\ni\sigma\mapsto \left\Vert {\mathcal L}(\sigma,f\upharpoonright\sigma)-f\right\Vert_1 \in \R 
\end{equation}
is measurable.

A learning rule $\mathcal L$ is {\em consistent} (with a function class $\mathscr F$) if for every $f\in {\mathscr F}$ and each $\sigma\in\Omega^n$ one has 
\[{\mathcal L}(\sigma,f\upharpoonright \sigma)\upharpoonright\sigma = f\upharpoonright\sigma.\]
In the case of a concept class $\mathscr C$, the consistency condition becomes this: for every $C\in {\mathscr C}$ and each $\sigma\in\Omega^n$ one has 
\[{\mathcal L}(\sigma,C\cap \sigma)\cap\sigma = C\cap\sigma.\]

A learning rule $\mathcal L$ is {\em probably approximately correct} ({\em PAC}) {\em under ${\mathcal P}$} if for every $\e>0$
\begin{equation}
\label{eq:fpac}
\sup_{\mu\in {\mathcal P}}\sup_{f\in {\mathscr F}}
\mu^{\otimes n}\left\{\sigma\in\Omega^n\colon \left\Vert {\mathcal L}(\sigma,f\upharpoonright\sigma)-f\right\Vert_1>\e \right\} \to 0\mbox{ as }n\to\infty.\end{equation}
Here $\mu^{\otimes n}$ denotes the (Lebesgue extension of the) product measure on $\Omega^n$. Now the origin of the measurability condition (\ref{eq:measurabilityl}) on the mapping $\mathcal L$ is clear: it is implicit in (\ref{eq:fpac}). 

Equivalently, there is a function $s(\e,\delta)$ ({\em sample complexity} of $\mathcal L$) such that for each $f\in{\mathscr F}$ and every $\mu\in {\mathcal P}$ an i.i.d. sample $\sigma$ with $\geq s(\e,\delta)$ points has the property $\left\Vert {\mathcal L}(\sigma,f\upharpoonright\sigma)-f\right\Vert_1<\e$ with confidence $\geq 1-\delta$.

In particular, for a concept class $\mathscr C$, it is convenient to rewrite the definition of a PAC learning rule thus: for each $\e>0$,
\begin{equation}
\label{eq:pac}
\sup_{\mu\in {\mathcal P}}\sup_{C\in {\mathscr C}}
\mu^{\otimes n}\left\{\sigma\in\Omega^n\colon \mu\left({\mathcal L(\sigma,C\cap\sigma)}\bigtriangleup C\right)>\e \right\} \to 0\mbox{ as }n\to\infty.\end{equation}

In terms of the sample complexity function $s(\e,\delta)$, a learning rule $\mathcal L$ is PAC if for each $C\in{\mathscr C}$ and every $\mu\in {\mathcal P}$ an i.i.d. sample $\sigma$ with $\geq s(\e,\delta)$ points has the property $\mu(C\bigtriangleup {\mathcal L}(\sigma,C\cap\sigma))<\e$ with confidence $\geq 1-\delta$.

A function class $\mathscr F$ is {\em PAC learnable} {\em under $\mathcal P$}, if there exists a PAC learning rule for $\mathscr F$ ($\mathscr C$) under $\mathcal P$. A class $\mathscr F$ is {\em consistently learnable} (under $\mathcal P$) if every learning rule consistent with $\mathscr F$ is PAC under $\mathcal P$. 
If $\mathcal P = P(\Omega)$ is the set of all probability measures, then $\mathscr F$ is said to be (distribution-free) {\em PAC learnable}. If $P=\{\mu\}$ is a single probability measure, one is talking of {\em learning under a single measure} (or distribution).
These definitions apply in particular to concept classes as well.
Learnability under intermediate families of measures on $\Omega$ has received considerable attention, cf. Chapter 7 in \cite{vidyasagar2003}.

Notice that in this paper, we only talk of {\em potential} PAC learnability, adopting a purely information-theoretic viewpoint. As a consequence, our statements about learning rules are existential rather than constructive, and 
building learning rules by transfinite recursion is perfectly acceptable.

An important concept 
is that of a {\em uniform Glivenko--Cantelli} function class {\em with respect to a family of measures} $\mathcal P$, that is, a function class $\mathscr F$ such that for each $\e>0$
\begin{equation}
\label{eq:fglivenko}
\sup_{\mu\in {\mathcal P}}\mu^{\otimes n}\left\{\sup_{f\in{\mathscr F}}\left\vert \E_{\mu}(f)-  \E_{\mu_n}(f) \right\vert\geq \e\right\}\to 0\mbox{ as }n\to\infty,
\end{equation}
(cf. \cite{dudley}, Ch. 3; \cite{mendelson:03}.) Here $\mu_n$ stands for the empirical (uniform) measure on $n$ points, sampled in an i.i.d. fashion from $\Omega$ according to the distribution $\mu$. The symbol $\E_{\mu_n}$ means the empirical mean of $f$ on the sample $\sigma$. One also says that $\mathscr F$ has the property of {\em uniform convergence of empirical means} ({\em UCEM} property) {\em with respect to} $\mathcal P$ \cite{vidyasagar2003}. 

In the case of a concept class $\mathscr C$, the uniform Glivenko--Cantelli property becomes
\begin{equation}
\label{eq:glivenko}
\sup_{\mu\in {\mathcal P}}\mu^{\otimes n}\left\{\sup_{C\in{\mathscr C}}\left\vert \mu(C)-  \mu_n(C) \right\vert\geq \e\right\}\to 0\mbox{ as }n\to\infty.
\end{equation}
In this case, one says that $\mathscr C$ has the property of {\em uniform convergence of empirical measures}, which is also abbreviated to UCEM property (with respect to $\mathcal P$).

Every uniform Glivenko--Cantelli class (with respect to $\mathcal P$) is PAC learnable (under $\mathcal P$). In the distribution-free situation the converse holds under mild additional measurability conditions on the class (but not always \cite{DD}, see a discussion in Section \ref{s:dd} below). 
For learning under a single measure, it is not so: a PAC learnable class under a single distribution $\mu$ need not be uniform Glivenko-Cantelli with respect to $\mu$ (cf. Chapter 6 in \cite{vidyasagar2003}, or else \cite{pestov:2010}, Example 2.10, where a countable counter-example is given). Not every PAC learnable class under non-atomic measures is uniform Glivenko--Cantelli with respect to non-atomic measures either: the class consisting of all finite and all cofinite subsets of $\Omega$ is a counter-example.

We say, following Pollard \cite{pollard}, that a function class $\mathscr F$ is {\em universally separable} if it contains a countable subfamily $\mathscr F^{\prime}$ which is {\em universally dense} in $\mathscr F$: every function $f\in \mathscr F$ is a pointwise limit of a sequence of elements of $\mathscr F^{\prime}$. By the Lebesgue Dominated Convergence Theorem, for every probability measure $\mu$ on $\Omega$ the set $\mathscr F^{\prime}$ is everywhere dense in $\mathscr F$ in the $L^1(\mu)$-distance. In particular, a concept class $\mathscr C$ is universally separable if it contains a countable subfamily $\mathscr C^\prime$ with the property that for every $C\in {\mathscr C}$ there exists a sequence $(C_n)_{n=1}^\infty$ of sets from $\mathscr C^\prime$ and for every $x\in\Omega$ there is $N$ with the property that, for all $n\geq N$, $x\in C_n$ if $x\in C$, and $x\notin C_n$ if $x\notin C$. 

Probably the main source of uniform Glivenko--Cantelli classes is the finiteness of VC dimension. Assume that $\mathscr C$ satisfies a suitable measurability condition, for instance, $\mathscr C$ is image admissible Souslin \cite{dudley}, or else universally separable. (In particular, a countable $\mathscr C$ satisfies either condition.) If $\VC({\mathscr C})=d <\infty$, then $\mathscr C$ is uniform Glivenko--Cantelli, with a sample complexity bound that does not depend on $\mathscr C$, but only on $\e$, $\delta$, and $d$. The following is a typical (and far from being optimal) such estimate, which can be deduced, for instance, along the lines of \cite{mendelson:03}:
\begin{equation}
\label{eq:standard}
s(\e,\delta,d)\leq
\frac{128}{\e^2}\left(d\log\left(\frac{2e^2}{\e}\log\frac{2e}{\e}\right) + \log\frac 8 {\delta}\right).
\end{equation}
For our purposes, we will fix any such bound and refer to it as a {\em ``standard''} sample complexity estimate for $s(\e,\delta,d)$. 

Let us recall a more general concept of fat shattering dimension \cite{ABDCBH} which is relevant for function classes. Let $\e>0$. A finite subset $A$ of $\Omega$ is {\em $\e$-fat shattered} by a function class $\mathscr F$ with witness function $h\colon A\to [0,1]$ if for every $B\subseteq A$ there is a function $f_B\in {\mathscr F}$ such that 
\begin{equation}
\label{eq:fb}
\begin{cases}f_B(a)>h(a)+\e&\mbox{ for }a\in B,\\
f_B(a)<h(a)-\e&\mbox{ for }a\in A\setminus B.
\end{cases}
\end{equation}
The {\em $\e$-fat shattering dimension} of $\mathscr F$ (over the domain $\Omega$) is defined as
\[{\mathrm{fat}}_{\e}{\mathscr F} = \sup\left\{\abs A\colon A\subseteq\Omega,~ A\mbox{ is $\e$-fat shattered by }\mathscr F \right\}.\]
In particular, if $\mathscr C$ is a concept class, then for any $\e\leq 1/2$ the $\e$-fat shattering dimension of $\mathscr C$ is the VC dimension of $\mathscr C$. If we want to stress that the combinatorial dimension is calculated over a particular domain $\Omega$, we will use the notation $\fat_{\e}({\mathscr F}\upharpoonright\Omega)$ and $\VC({\mathscr C}\upharpoonright\Omega)$.

In the definition of $\e$-fat shattering dimension, one can assume without loss of generality the values of $\e$ and of a witness function to be rational. More precisely, the following holds.

\begin{lemma}
\label{l:rational}
Suppose a finite set $A$ is $\e$-fat shattered by a function class $\mathscr F$. Then there is a rational value $\e^\prime>\e$ such that $A$ is $\e^{\prime}$-fat shattered by $\mathscr F$ with a rational-valued witness function $h^{\prime}\colon A\to\Q$.
\end{lemma}

\begin{proof}
Let $h$ be a witness of $\e$-fat shattering for $A$.
For each $B\subseteq A$ choose a function $f_B$ satisfying Condition (\ref{eq:fb}). For every $a\in A$ define
\[S_a = \min_{a\in B}f_B(a),~~s_a = \max_{a\in A\setminus B} f_B(a).\]
One has: $s_a<h(a)-\e < h(a)+\e < S_a$, and so $S_a-s_a>2\e$. One can therefore select rational values $\e^{\prime}_a>\e$ and $h^{\prime}(a)$ such that $s_a+\e^{\prime}_a < h^{\prime}(a) < S_a-\e^{\prime}_a$. This way, we obtain a desired witness function $h^\prime$, and the proof is now finished by posing $\e^\prime=\min_{a\in A}\e^{\prime}_a$.
\end{proof}

Every function class $\mathscr F$ whose $\e$-fat shattering dimension is finite at every scale $\e>0$ is uniform Glivenko--Cantelli. Here is an asymptotic estimate of the sample size taken from \cite{ABDCBH} (Theorem 3.6):
\begin{equation}
\label{eq:asymptoticest}
s(\e,\delta,d) \leq C\left(\frac{1}{\e^2}{d(\e/24)}({\mathscr F})\log^2\frac{d(\e/24)}{\e} +\log\frac{1}{\delta}\right),
\end{equation}
where $d\colon \R_+\to\N$ is the fat-shattering dimension of $\mathscr F$ understood as a function of epsilon, $d(\e)=\fat_{\e}({\mathscr F})$. In the formula, $C$ denotes a universal constant whose value can be extracted from the proofs in \cite{ABDCBH}, but, given the presence of such a loose scale as $\e/24$, does not really matter. Tighter sample size estimates can be found in \cite{BL}. Again, we will refer to Condition (\ref{eq:asymptoticest}) as {\em ``standard''} complexity estimate corresponding to the fat shattering dimension function $d$.

Finally, recall that a subset $N\subseteq\Omega$ is {\em universal null} if for every non-atomic probability measure $\mu$ on $(\Omega,{\mathscr A})$ one has $\mu(N^\prime)=0$ for some Borel set $N^\prime$ containing $N$. Universal null Borel sets are just countable sets. 

\section{Revisiting an example of Durst and Dudley\label{s:dd}}

In order to explain our approach to constructing a learning rule that is PAC under non-atomic distributions, we need to examine the traditional way of proving distribution-free PAC learnability. A usual approach consists of two stages.

1. A function (or concept) class $\mathscr F$ is uniform Glivenko--Cantelli as long as a suitable combinatorial parameter of $\mathscr F$ (VC dimension, fat-shattering dimension etc.) is finite.

2. A uniform Glivenko--Cantelli class $\mathscr F$ is PAC learnable. Moreover, such a class is {\em consistently} PAC learnable: every consistent learning rule $\mathcal L$ for $\mathscr F$ is probably approximately correct.

The proof of every statement of the former type depends in an essential way on the Fubini theorem, and so some measurability restrictions on the class $\mathscr F$ are necessary. Without them, the conclusion is not true in general. Here is a classical example of a concept class having finite VC dimension which is not uniform Glivenko--Cantelli.

\begin{example}[Durst and Dudley \cite{DD}, Proposition 2.2; cf. also \cite{WD}, p. 314; \cite{dudley}, pp. 170--171]
\label{ex:dd} Let $\Omega$ be an uncountable standard Borel space, that is, up to an isomorphism, a Borel space associated to the unit interval $[0,1]$. The cardinality of $\Omega$ is continuum. Choose a minimal well-ordering $\prec$ on $\Omega$, and let $\mathscr C$ consist of all half-open initial segments of the ordered set $(\Omega,\prec)$, that is, subsets of the form $I_y=\{x\in\Omega\colon x\prec y\}$, $y\in\Omega$. Clearly, the VC dimension of the class $\mathscr C$ is one. 

Fix a non-atomic Borel probability measure $\mu$ on $\Omega$ (e.g., the Lebesgue measure on $[0,1]$). 

Now assume the validity of the Continuum Hypothesis. 
Under this assumption, every element of $\mathscr C$ is a countable set, therefore Borel measurable of measure zero. 
At the same time, for every $n$ and each random $n$-sample $\sigma$, there is a countable initial segment $C\in {\mathscr C}$ containing all elements of $\sigma$. The empirical measure of $C$ with respect to $\sigma$ is one. Thus, no finite sample guesses the measure of all elements of $\mathscr C$ to within an accuracy $\e<1$ with a non-vanishing confidence.
\end{example}

A further modification of this construction gives an example of a concept class of finite VC dimension which is not consistently PAC learnable.

\begin{example}[Blumer, Ehrenfeucht, Haussler, and Warmuth \cite{BEHW}, p. 953]
\label{ex:behw}
Again, assume the Continuum Hypothesis.
Add to the concept class $\mathscr C$ from Example \ref{ex:dd} the set $\Omega$ as an element. In other words, form a concept class $\mathscr C^\prime$ consisting of all intitial segments of $(\Omega,\prec)$, including improper ones. One still has $\VC(\mathscr C^\prime)=1$.
For a finite labelled sample $(\sigma,\tau)$ define
\begin{equation}
\label{eq:rule}
{\mathcal L}(\sigma,\tau) = \min\{y\colon \tau\subseteq I_y\}.\end{equation}
The learning rule $\mathcal L$ is clearly consistent with the class $\mathscr C$, but is not probably approximately correct, because for the concept $C=\Omega$ the value ${\mathcal L}(\Omega\cap\sigma)={\mathcal L}(\sigma,\sigma)$ will always return a countable concept $I_y$, and if $\mu$ is a non-atomic Borel probability measure on $\Omega$, then $\mu(C\bigtriangleup I_y)=1$. The concept $C=\Omega$ is not learned to accuracy $\e<1$ with a non-zero confidence.
\end{example}

\begin{remark}
It is important to note that --- again, under the Continuum Hypothesis --- the class ${\mathscr C}^\prime$ is nevertheless distribution-free PAC learnable. 
\end{remark}

Indeed, redefine a well-ordering on ${\mathscr C}^\prime=\{I_x\colon x\in\Omega\}\cup\{\Omega\}$ by making $\Omega$ the smallest element (instead of the largest one) and keeping the order relation between the other elements the same. Denote the new order relation by $\prec_1$, and define a learning rule $\mathcal L_1$ similarly to Eq. (\ref{eq:rule}), but this time understanding the minimum with respect to the order $\prec_1$:
\begin{equation}
\label{eq:rule1}
{\mathcal L}_1(\sigma,\tau) = \min_{(\prec_1)}\left\{C\in{\mathscr C}^{\prime}\colon C\cap\sigma = \bigcap_{\tau\subseteq D}D\right\}.\end{equation}
In essence, $\mathcal L_1$ examines all the concepts following a transfinite order on them, and if a labelled sample is consistent with the class ${\mathscr C}^\prime$, then $\mathcal L_1$ returns the first concept consistent with the sample that it comes across.

To understand what difference it makes with Example \ref{ex:behw}, let $\mu$ be again a non-atomic probability measure on $\Omega$. If $C=\Omega$, then for every sample $\sigma$ consistently labelled with $C$ the rule $\mathcal L_1$ will return $C$, because this is the smallest consistent concept encountered by the algorithm. If $C\neq\Omega$, then for $\mu$-almost all samples $\sigma$ (that is, for a set of $\mu$-measure one) the labelling on $\sigma$ produced by $C$ will be empty, and the concept ${\mathcal L}_1(\sigma,\emptyset)$ returned by $\mathcal L_1$, while possibly different from $C$, will be again a countable concept, meaning that $\mu(C\bigtriangleup{\mathcal L}(\sigma,\emptyset))=0$. 

To give a formal proof that $\mathcal L_1$ is PAC, notice that for every $C\in {\mathscr C}^\prime$ and each $n\in\N$ the collection of pairwise distinct concepts ${\mathcal L}_1(\sigma\cap C)$, $\sigma\in\Omega^n$ is only countable (under Continuum Hypothesis), because they are all contained in the $\prec_1$-initial segment of a minimally ordered set ${\mathscr C}^\prime$ of cardinality continuum, bounded by $C$ itself. As a consequence, the concept class
\begin{equation}
\label{eq:lc}
{\mathcal L}_1^{C}=\{{\mathcal L}_1(\sigma\cap C)\colon\sigma\in\Omega^n,n\in\N\}\subseteq {\mathscr C}^\prime\end{equation}
is also countable (assuming Continuum Hypothesis). The VC dimension of the family ${\mathcal L}_1^{C}\cup\{C\}$ is $\leq 1$, and being countable, it is a uniform Glivenko--Cantelli class with a standard sample complexity as in Eq. (\ref{eq:standard}). Consequently, given $\e,\delta>0$, and assuming that $n$ is sufficiently large, one has for each probability measure $\mu$ on $\Omega$ and every $\sigma\in\Omega^n$
\[\mu(C\bigtriangleup{\mathcal L}(\sigma,C\cap\sigma))<\e\]
provided $n\geq s(\e,\delta,1)$,
as required.

\begin{remark}
\label{r:notalways}
Thus, under the Continuum Hypothesis, the example of Dudley and Durst as modified by Blumer, Ehrenfeucht, Haussler, and Warmuth gives an example of a PAC learnable concept class which is not uniform Glivenko--Cantelli (even if having finite VC dimension). As it will become clear in the next Section, the assumption of Continuum Hypothesis can be weakened to Martin's Axiom. Still, it would be interesting to know whether an example with the same combination of properties can be constructed without additional set-theoretic assumptions.
\end{remark}

A basic observation of this section is that in order for a learning rule $\mathcal L$ to be PAC, the assumption on $\mathscr F$ being uniform Glivenko--Cantelli can be weakened as follows.

\begin{lemma}
\label{l:basic}
Let $\mathscr F$ be a function class and $\mathcal P$ a family of probability measures on the domain $\Omega$. Suppose there exists a function $s(\e,\delta)$ and a consistent learning rule $\mathcal L$ for $\mathscr F$ with the property that for every $f\in{\mathscr F}$, the set ${\mathcal L}^{f}\cup\{f\}$ is Glivenko--Cantelli with respect to $\mathcal P$ with the sample complexity $s(\e,\delta)$, where
\[{\mathcal L}^{f}=\left\{{\mathcal L}(f\upharpoonright\sigma)\colon\sigma\in\Omega^n,n\in\N\right\}.\]
Then $\mathcal L$ is probably approximately correct under $\mathcal P$ with sample complexity $s(\e,\delta)$. \qed
\end{lemma}

\begin{remark}
Of course instead of ${\mathcal L}^{f}\cup\{f\}$ it is sufficient to make the same assumption on the class ${\mathcal L}^{f}$. This will not affect the PAC learnability of $\mathcal L$. However, an estimate for the sample complexity of the union in terms of $s(\e,\delta)$ will be somewhat awkward, and in view of a specific way in which the above Lemma is going to be used, the current assumption is technically more convenient.
\end{remark}

This simple fact becomes very useful in combination with the technique of well-orderings in the case where $\mathcal P$ consists of non-atomic measures and therefore consistent PAC learnability is not to be expected. At the same time, this approach requires additional set-theoretic axioms in order to assure measurability of emerging function classes.
Of course the Continuum Hypothesis is a rather strong assumption, which is particularly unnatural in a probabilistic context (cf. \cite{freiling}). But it is unnecessary.
Martin's Axiom is a much weaker and natural additional set-theoretic axiom, which works just as well. We explain how the above idea is formalized in the setting of Martin's Axiom in the next Section.

\section{Learnability under Martin's Axiom}

Martin's Axiom (MA) \cite{fremlin,jech,kunen} in one of its equivalent forms says that no compact Hausdorff topological space with the countable chain condition is a union of strictly less than continuum nowhere dense subsets. Thus, it can be seen as a strengthening of the statement of the Baire Category Theorem. In particular, the Continuum Hypothesis (CH) implies MA. However, MA is compatible with the negation of CH, and this is where the most interesting applications of MA are to be found. We will be using just one particular consequence of Martin's Axiom. For the proof of the following result, see \cite{kunen}, Theorem 2.21, or \cite{fremlin}, or \cite{jech}, pp. 563--565.

\begin{theorem}[Martin-Solovay] 
\label{th:martin-solovay}
Let $(\Omega,\mu)$ be a standard Lebesgue non-atomic probability space. 
Under Martin's Axiom, 
the Lebesgue measure is $2^{\aleph_0}$-additive, that is, if $\kappa<2^{\aleph_0}$ and $A_{\alpha}$, $\alpha<\kappa$ is family of pairwise disjoint measurable sets, then $\cup_{\alpha<\kappa}A_{\alpha}$ is Lebesgue measurable and
\[\mu\left(\bigcup_{\alpha<\kappa}A_{\alpha} \right) = \sum_{\alpha<\kappa}\mu(A_{\alpha}).\]
In particular, the union of less than continuum null subsets of $\Omega$ is a null subset.
\qed
\end{theorem}

Here is a central technical tool used in our proofs.

\begin{lemma}
\label{l:fma}
Let $\mathscr F$ be a function class  and $\mathcal P$ a family of probability measures on a standard Borel domain $\Omega$. Consider the following properties.
\begin{enumerate}
\item \label{fma:1} Every countable subclass of $\mathscr F$ is uniform Glivenko--Cantelli with respect to $\mathcal P$.
\item \label{fma:2} There is a function $s(\e,\delta)$ such that every countable subclass of $\mathscr F$ is uniform Glivenko--Cantelli with respect to $\mathcal P$ with sample complexity $s(\e,\delta)$.
\item \label{fma:3} Every subclass ${\mathscr F}^\prime$ of $\mathscr F$ having cardinality $<2^{\aleph_0}$ is uniform Glivenko--Cantelli with respect to $\mathcal P$.
\item \label{fma:4} There is a function $s(\e,\delta)$ such that every subclass ${\mathscr F}^\prime$ of $\mathscr F$ having cardinality $<2^{\aleph_0}$ is uniform Glivenko--Cantelli with respect to $\mathcal P$ with sample complexity $s(\e,\delta)$.
\end{enumerate}
Then  
\begin{center}
(\ref{fma:1}) \\
$\neswarrow\phantom{xx}\nwarrow$ \\
(\ref{fma:2}) $\phantom{xxxxx}$ (\ref{fma:3})\\
$\nwarrow\phantom{xx}\nearrow$ \\
(\ref{fma:4})
\end{center}
Under Martin's Axiom, all four conditions are equivalent.
\end{lemma}

\begin{proof}
The implications $(\ref{fma:2})\Rightarrow (\ref{fma:1})$, $(\ref{fma:3})\Rightarrow (\ref{fma:1})$, $(\ref{fma:4})\Rightarrow (\ref{fma:2})$ and $(\ref{fma:4})\Rightarrow (\ref{fma:3})$ are trivially true. To show $(\ref{fma:1})\Rightarrow (\ref{fma:2})$, let $\delta,\e>0$ be arbitrary but fixed. For each countable subclass ${\mathscr F}^\prime$, choose the smallest value of sample complexity $s=s({\mathscr F}^\prime,\e,\delta)\in\N$. The integer-valued function ${\mathscr F}^\prime\mapsto s({\mathscr F}^\prime,\e,\delta)$ is monotone under inclusions: if ${\mathscr F}^\prime\subseteq {\mathscr F}^{\prime\prime}$, then $s({\mathscr F}^\prime,\e,\delta)\leq s({\mathscr F}^{\prime\prime},\e,\delta)$. If ${\mathscr F}^\prime_n$ is a countable sequence of countable classes, then the union $\cup_{n=1}^\infty {\mathscr F}^\prime_n$ is a countable class, whose sample complexity $s\left(\cup_{n=1}^\infty {\mathscr F}^\prime_n,\e,\delta\right)$ forms an upper bound for all $s({\mathscr F}^\prime,\e,\delta)$, $n=1,2,\ldots$. Thus, the function ${\mathscr F}^\prime\mapsto s({\mathscr F}^\prime,\e,\delta)$ for $\delta,\e>0$ fixed is bounded on countable sets of inputs. To conclude the proof, it is enough to notice that a real-valued function is bounded if and only if its restriction to every countable subset of the domain is bounded. 

Now assume Martin's Axiom. It is enough to prove $(\ref{fma:2})\Rightarrow (\ref{fma:4})$. 
This is done by a transfinite induction on the cardinality $\kappa=\abs{{\mathscr F}^\prime}<2^{\aleph_0}$. Let us pick the same complexity function $s=s(\e,\delta)$ as in $(\ref{fma:2})$.
For $\kappa=\aleph_0$ there is nothing to prove. Else, represent $\mathscr F$ as a union of an increasing transfinite chain of function classes ${\mathscr F}_{\alpha}$, $\alpha<\kappa$, for each of which the statement of (\ref{fma:4}) holds. For every $\e>0$ and $n\in\N$, the set 
\[\left\{\sigma\in\Omega^n\colon \sup_{f\in{\mathscr F}}\left\vert\E_{\mu_n(\sigma)}(f)-\E_\mu(f)\right\vert <\e\right\} = \bigcap_{\alpha<\kappa} \left\{\sigma\in\Omega^n\colon \sup_{f\in{\mathscr F}_{\alpha}}\left\vert\E_{\mu_n(\sigma)}(f)-\E_\mu(f)\right\vert <\e\right\}
\]
is measurable as an easy consequence of Martin-Solovay's Theorem \ref{th:martin-solovay}. Given $\delta>0$ and $n\geq s(\e,\delta)$, another application of the same result leads to conclude that for every $\mu\in P(\Omega)$:
\begin{eqnarray*}
\mu^{\otimes n}\left\{\sigma\in\Omega^n\colon \sup_{f\in{\mathscr F}}\left\vert\E_{\mu_n(\sigma)}(f)-\E_\mu(f)\right\vert <\e\right\} &=& 
\mu^{\otimes n}\left(
\bigcap_{\alpha<\kappa} \left\{\sigma\in\Omega^n\colon \sup_{f\in{\mathscr F}_{\alpha}}\left\vert\E_{\mu_n(\sigma)}(f)-\E_\mu(f)\right\vert <\e\right\}\right) \\
&=& \inf_{\alpha<\kappa} \mu^{\otimes n}\left\{\sigma\in\Omega^n\colon \sup_{f\in{\mathscr F}_{\alpha}}\left\vert\E_{\mu_n(\sigma)}(f)-\E_\mu(f)\right\vert <\e\right\}\\
&\geq& 1-\delta,
\end{eqnarray*}
as required.
\end{proof}

\begin{lemma}
\label{l:mapac}
Let $\mathscr F$ be a function class whose countable subclasses are uniform Glivenko--Cantelli with respect to a family of probability measures $\mathcal P$. Let $\mathcal L$ be a consistent learning rule for $\mathscr F$ with the property that for every $f\in{\mathscr F}$, the set
\begin{eqnarray}
\label{eq:l}
{\mathcal L}^{f,n}=\left\{{\mathcal L}(f\vert\sigma)\colon\sigma\in\Omega^n\right\}\end{eqnarray}
has cardinality strictly less than continuum. Under Martin's Axiom, the rule $\mathcal L$ is probably approximately correct under $\mathcal P$. The common sample complexity of countable subclasses of $\mathscr F$ becomes the sample complexity bound for the learning rule $\mathcal L$. 
\end{lemma}

\begin{proof}
Recall that $2^{\aleph_0}$ is a regular cardinal, and thus admits no countable cofinal subset. Therefore, under the assumptions of Lemma, the cardinality of ${\mathcal L}^{f}=\cup_{n=1}^{\infty}{\mathcal L}^{f,n}$ is still strictly less than continuum. The same is true of the class ${\mathcal L}^{f}\cup\{f\}$. Applying now Lemma \ref{l:fma} and then Lemma \ref{l:basic}, we conclude.
\end{proof}

The following result establishes existence of learning rules with the above property.

\begin{lemma}
\label{l:l}
Let $\mathscr F$ be an infinite function class on a measurable space $\Omega$. Denote $\kappa=\abs{\mathscr F}$ the cardinality of $\mathscr F$. There exists a consistent learning rule $\mathcal L$ for $\mathscr F$ with the property that for every $f\in {\mathscr F}$ and each $n$, the set ${\mathcal L}^{f,n}$ (cf. Eq. (\ref{eq:l}))
has cardinality $<\kappa$. Under Martin's Axiom the rule $\mathcal L$ satisfies the measurability condition (\ref{eq:measurabilityl}).
\end{lemma}

\begin{proof}
Choose a minimal well-ordering of elements of $\mathscr F$:
\[{\mathscr F}=\{f_\alpha\colon\alpha<\kappa\}.\]
Notice that $\kappa$ never exceeds the cardinality of the continuum $2^{\aleph_0}$ because ${\mathscr F}$ consists of Borel subsets of a standard Borel domain. For this reason, every initial segment of the above ordering has cardinality strictly less than $2^{\aleph_0}$.
For every $\sigma\in\Omega^n$ and $\tau\in [0,1]^n$, set the value
${\mathcal L}(\sigma,\tau)$ of the learning rule equal to $f_\beta$, where 
\[\beta = \min\{\alpha<\kappa\colon f_{\alpha}\vert\sigma=\tau\},\]
provided such a $\beta$ exists. 
Clearly, for each $\alpha<\kappa$ one has
\[{\mathcal L}(\sigma,f_{\alpha}\upharpoonright\sigma)\subseteq
\{f_\beta\colon\beta\leq\alpha\},\]
which assures that the set in (\ref{eq:l}) has cardinality strictly less than continuum. Besides, the learning rule $\mathcal L$ is consistent.

Fix $f=f_{\alpha}\in {\mathscr F}$, $\alpha<\kappa$. For every $\beta\leq\alpha$ define $D_{\beta}=\{\sigma\in\Omega^n\colon f\vert\sigma = f_{\beta}\vert\sigma\}$. The sets $D_{\beta}$ are measurable, and the function \[\Omega^n\ni\sigma\mapsto \E_{\mu}({\mathcal L}(f\upharpoonright\sigma)-f)\in\R\]
takes a constant value $\norm{f-f_{\beta}}_{L^1(\mu)}$ on each set $D_{\beta}\setminus\cup_{\gamma<\beta}D_{\gamma}$, $\beta\leq\alpha$. Such sets, as well as all their possible unions, are measurable under Martin's Axiom by force of Martin--Solovay's Theorem \ref{th:martin-solovay}, and their union is $\Omega^n$. This implies the condition (\ref{eq:measurabilityl}) for $\mathcal L$.
\end{proof}

Lemma \ref{l:mapac} and lemma \ref{l:l} lead to the following result.

\begin{theorem}[Assuming Martin's Axiom]
\label{th:countablesubclassesugc}
Let $\mathscr F$ be a function class consisting of Borel measurable functions on a standard Borel domain $\Omega$, and let $\mathcal P$ be a family of probability measures on $\Omega$. Suppose that every countable subclass of $\mathscr F$ is uniform Glivenko--Cantelli with respect to $\mathcal P$. Then the function class $\mathscr F$ is PAC learnable. In addition, there exists a common sample complexity bound for countable subclasses of $\mathscr F$, and any such bound  gives a sample complexity bound for PAC learnability of $\mathscr F$. \qed
\end{theorem}

We again recall that a set $A\subseteq\Omega$ is {\em universal null} if it is Lebesgue measurable with respect to every non-atomic Borel probability measure $\mu$ on $\Omega$ and $\mu(A)=0$. 
 
\begin{corollary}[Assuming Martin's Axiom]
\label{l:fsd}
Let $\mathscr F$ be a function class consisting of Borel measurable functions on a standard Borel space $\Omega$. Suppose for every $\e>0$ there is a natural number $d(\e)$ such
that every countable subclass ${\mathscr F}^\prime\subseteq {\mathscr F}$ has $\e$-fat shattering dimension $\leq d(\e)$ outside of some universal null set (which depends on ${\mathscr F}^{\prime}$). Then the function class $\mathscr F$ is PAC learnable under the family $\mathcal P$ of non-atomic probability measures, with the standard sample complexity corresponding to the given value of fat shattering dimension.
\end{corollary}

\begin{proof}
Let ${\mathscr F}^\prime\subseteq {\mathscr F}$ be a countable subclass. For every $n\in\N$, choose a null set $A_n$ such that the $\e$-fat shattering dimension of ${\mathscr F}^\prime$ restricted to $\Omega\setminus A_n$ is bounded by $d(1/n)$. Consider $A=\cup_{n=1}^{\infty}A_n$. The function class ${\mathscr F}^\prime$ restricted to $\Omega\setminus A$ is uniform Glivenko--Cantelli, with the usual sample complexity given by $d(\e)$. In particular, ${\mathscr F}^\prime\vert\Omega\setminus A$ is uniform Glivenko--Cantelli with respect to the family $\mathcal P$ of non-atomic probability measures. Since $\mu(A)=0$ for all $\mu\in{\mathcal P}$, we conclude that the class ${\mathscr F}^\prime$ is uniform Glivenko--Cantelli with respect to $\mathcal P$ even if viewed on the original domain of definition, $\Omega$.
\end{proof}

\begin{corollary}[Assuming Martin's Axiom]
\label{l:vcd}
Let $\mathscr C$ be a concept class consisting of Borel measurable functions on a standard Borel space $\Omega$. Suppose that for some $d$ every countable subclass ${\mathscr C}^\prime\subseteq {\mathscr C}$ has VC dimension $\leq d$ outside of a universal null set (which depends on ${\mathscr C}^{\prime}$). Then the concept  class $\mathscr C$ is PAC learnable under the family $\mathcal P$ of non-atomic probability measures, with the standard sample complexity corresponding to the given value of VC dimension.  \qed
\end{corollary}

\section{VC dimension and Boolean algebras}

Recall that a {\em Boolean algebra}, $B=\langle B,\wedge,\vee,\neg,0,1 \rangle$, consists of a set, $B$, equipped with two associative and commutative binary operations, $\wedge$ (``meet'') and $\vee$ (``join''), which are distributive over each other and satisfy the absorption principles $a\vee (a\wedge b)=a$, $a\wedge (a\vee b)=a$, as well as
a unary operation $\neg$ (complement) and two elements $0$ and $1$, satisfying 
$a\vee \neg a =1$, $a\wedge \neg a=0$. 

For instance, the family $2^\Omega$ of all subsets of a set $\Omega$, with the union as join, intersection as meet, the empty set as $0$ and $\Omega$ as $1$, as well as the set-theoretic complement $\neg A = A^c$, forms a Boolean algebra. In fact, every Boolean algebra can be realized as an algebra of subsets of a suitable $\Omega$. Even better, according to the Stone representation theorem, a Boolean algebra $B$ is isomorphic to the Boolean algebra formed by all open-and-closed subsets of a suitable compact space, $S(B)$, called the {\em Stone space} of $B$, where the Boolean algebra operations are interpreted set-theoretically as above. 

The space $S(B)$ can be obtained in different ways. For instance, one can think of elements of $S(B)$ as Boolean algebra homomorphisms from $B$ to the two-element Boolean algebra $\{0,1\}$ (the algebra of subsets of a singleton). In this way, $S(B)$ is a closed topological subspace of the compact zero-dimensional space $\{0,1\}^B$ with the usual Tychonoff product topology.

The Stone space of the Boolean algebra $B=2^\Omega$ is known as the {\em Stone-\v Cech compactification of $\Omega$}, and is denoted $\beta\Omega$. 
The elements of $\beta\Omega$ are ultrafilters on $\Omega$. A collection $\xi$ of non-empty subsets of $\Omega$ is an {\em ultrafilter} if it is closed under finite intersections and if for every subset $A\subseteq\Omega$ either $A\in\xi$ or $A^c\in\xi$. To every point $x\in\Omega$ there corresponds a {\em trivial} ({\em principal}) {\em ultrafilter,} $\bar x$, consisting of all sets $A$ containing $x$. However, if $\Omega$ is infinite, the Axiom of Choice assures that there exist {\em non-principal} ultrafilters on $\Omega$. Recall that a non-empty family $\Phi$ of non-empty subsets of a set $X$ is a {\em filter} if it is closed under finite intersections and supersets. An equivalent form of the Axiom of Choise states that every filter is contained in an ultrafilter. Now starting with a filter having an empty intersection (e.g. the filter of all cofinite subsets of the natural numbers), one obtained a non-principal ultrafilter.

Basic open sets in the space $\beta\Omega$ are of the form $\bar A = \{\zeta\in\beta\Omega\colon A\in\zeta\}$, where $A\subseteq\Omega$. It is interesting to note that each $\bar A$ is at the same time closed, and in fact $\bar A$ is the closure of $A$ in $\beta\Omega$. Moreover, every open and closed subset of $\beta\Omega$ is of the form $\bar A$.

A one-to-one correspondence between ultrafilters on $\Omega$ and Boolean algebra homomorphisms $2^{\Omega}\to \{0,1\}$ is this: think of an ultrafilter $\xi$ on $\Omega$ as its own indicator function $\chi_\xi$ on $2^{\Omega}$, sending $A\subseteq\Omega$ to $1$ if and only if $A\in\xi$. It is not difficult to verify that $\chi_\xi$ is a Boolean algebra homomorphism, and that every homomorphism arises in this way.

The book \cite{johnstone} is a standard reference to the above topics.

Given a subset $\mathscr C$ of a Boolean algebra $B$, and a subset $X$ of the Stone space $S(B)$, one can regard $\mathscr C$ as a set of binary functions restricted to $X$, and compute the VC dimension of $\mathscr C$ over $X$. We will denote this parameter $\VC({\mathscr C}\upharpoonright X)$. 

A subset $I$ of a Boolean algebra $B$ is an {\em ideal} if, whenever $x,y\in I$ and $a\in B$, one has $x\vee y\in I$ and $a\wedge x\in I$. Define a {\em symmetric difference} on $B$ by the formula $x\bigtriangleup y =(x\vee y)\wedge\neg(x\wedge y)$. 
The {\em quotient Boolean algebra} $B/I$ consists of all equivalence classes modulo the equivalence relation $x\sim y\iff x\bigtriangleup y\in I$. It can be easily verified to be a Boolean algebra on its own, with operations induced from $B$ in a unique way. 

The Stone space of $B/I$ can be identified with a compact topological subspace of $S(B)$, consisting of all homomorphisms $B\to \{0,1\}$ whose kernel contains $I$. For instance, if $B=2^{\Omega}$ and $I$ is an ideal of subsets of $\Omega$, then the Stone space of $2^{\Omega}/I$ is easily seen to consist of all ultrafilters on $\Omega$ which do not contain sets from $I$.

\begin{theorem}
\label{th:shatter}
Let $\mathscr C$ be a concept class consisting of measurable subsets of a measurable domain $\Omega=(\Omega,{\mathscr A})$, and let $I$ be an ideal of sets on $\Omega$. The following conditions are equivalent.
\begin{enumerate}
\item 
\label{i:vc}
The $VC$ dimension of the (family of closures of the) concept class $\mathscr C$ restricted to the Stone space of the quotient algebra $2^\Omega/I$ is at least $n$: $\VC({\mathscr C}\upharpoonright S(2^\Omega/I))\geq n$.
\item 
\label{i:shattered}
There exists a family $A_1,A_2,\ldots,A_n$ of subsets of $\Omega$ not belonging to $I$, which is shattered by $\mathscr C$ in the sense that if $J\subseteq \{1,2,\ldots,n\}$, then there is $C\in {\mathscr C}$ which contains all sets $A_i$, $i\in J$, and is disjoint from all sets $A_i$, $i\notin J$. 
In addition, the subsets $A_i$ can be assumed measurable.
\end{enumerate}
\end{theorem}

\begin{proof}
(\ref{i:vc})$\Rightarrow$(\ref{i:shattered}). Choose ultrafilters $\xi_1,\ldots,\xi_n$ in the Stone space of the Boolean algebra $2^\Omega/I$, whose collection is shattered by $\mathscr C$. 
For every $J\subseteq \{1,2,\ldots,n\}$, select $C_J\in {\mathscr C}$ which carves the subset $\{\xi_i\colon i\in J\}$ out of $\{\xi_1,\ldots,\xi_n\}$. This means ${C_J}\in\xi_i$ if and only if $i\in J$. For all $i=1,2,\ldots,n$, set
\begin{equation}
\label{eq:bigcap}
A_i = \bigcap_{J\ni i}C_J\cap \bigcap_{J\not\ni i}C_J^c.
\end{equation}
Then $A_i\in\xi_i$ and hence $A_i\notin I$. Furthermore, if $i\in J$, then clearly $A_i\subseteq C_J$, and if $i\notin J$, then $A_i\cap C_J=\emptyset$.
The sets $A_i$ are measurable by their definition.
\par
(\ref{i:shattered})$\Rightarrow$(\ref{i:vc}). Let $A_1,A_2,\ldots,A_n$ be a family of subsets of $\Omega$ not belonging to the set ideal $I$ and shattered by $\mathscr C$ in sense of the lemma. For every $i$, the family of sets of the form $A_i\cap B^c$, $B\in I$ is a filter and so is contained in some ultrafilter $\xi_i$, which is clearly disjoint from $I$ and contains $A_i$. 
If $J\subseteq \{1,2,\ldots,n\}$ and $C_J\in{\mathscr C}$ contains all sets $A_i$, $i\in J$ and is disjoint from all sets $A_i$, $i\notin J$, then the closure $\bar C_J$ of $C_J$ in the Stone space contains $\xi_i$ if and only if $i\in J$. We conclude: the collection of ultrafilters $\xi_i$, $i=1,2,\ldots,n$, which are all contained in the Stone space of $2^{\Omega}/I$, is shattered by the closed sets $\bar C_J$.
\end{proof}

It follows in particular that the VC dimension of a concept class does not change if the domain $\Omega$ is compactified.

\begin{corollary}
$\VC({\mathscr C}\upharpoonright\Omega)=\VC({\mathscr C}\upharpoonright\beta\Omega)$.
\end{corollary}

\begin{proof}
The inequality $\VC({\mathscr C}\upharpoonright\Omega)\leq\VC({\mathscr C}\upharpoonright\beta\Omega)$ is trivial. To establish the converse, assume there is a subset of $\beta\Omega$ of cardinality $n$ shattered by $\mathscr C$.
Choose sets $A_i$ as in Theorem \ref{th:shatter},(\ref{i:shattered}). Clearly, any subset of $\Omega$ meeting each $A_i$ at exactly one point is shattered by $\mathscr C$.
\end{proof}

\begin{definition}
Given a concept class $\mathscr C$ on a domain $\Omega$ and an ideal $I$ of subsets of $\Omega$, we define the VC dimension of $\mathscr C$ modulo $I$, 
\[\VC({\mathscr C}\,{\mathrm{mod}}\,I) = \VC({\mathscr C}\upharpoonright S(2^\Omega/I)).\]
That is, $\VC({\mathscr C}\,{\mathrm{mod}}\,I)\geq n$ if and only if any of the equivalent conditions of Theorem \ref{th:shatter} are met. 
\end{definition}

\begin{definition}
Let $\mathscr C$ be a concept class on a domain $\Omega$. If $I$ is the ideal of all countable subsets of $\Omega$, we denote the $\VC({\mathscr C}\,{\mathrm{mod}}\,I)$ by $\VC({\mathscr C}\modd{\omega_1})$ and call it the {\em VC dimension modulo countable sets}.
\end{definition}

Now Theorem \ref{th:shatter} validates a definition of VC dimension modulo countable sets in a form stated in Introduction to our article.

\section{Fat-shattering dimension modulo countable sets}

When dealing with real-valued functions instead of subsets of the domain, the role of Boolean algebras is taken over by commutative $C^\ast$-algebras. Here is a brief summary. See e.g. \cite{arveson} for more.

Recall that a $C^\ast$-algebra is an associative algebra over the field of complex numbers $\C$ equipped with an involution (an anti-linear map $x\mapsto x^\ast$) and a norm which is submultiplicative ($\norm{xy}\leq\norm x\norm y$) and satisfies the property $\norm{x^\ast x}=\norm x^2$. For instance, the family $C(X)$ of all continuous complex-valued functions on a compact topological space $X$ forms a commutative unital $C^\ast$-algebra. Conversely, every commutative unital $C^\ast$-algebra $A$ is of this form. The space $X$, called the {\em Gelfand space}, or the {\em maximal ideal space} of $A$, is uniquely defined. Its elements can be described as non-zero multiplicative complex linear functionals on $A$. The topology on the space of such functionals is the weak star (weak$^\ast$) topology, that is, the coarsest topology making every evaluation map $f\mapsto f(a)$, $a\in A$, continuous. 

We want to calculate the maximal ideal space of the $C^\ast$-algebra $\ell^{\infty}(\Omega)$ of all bounded complex-valued functions on a set $\Omega$. With this purpose, we introduce the following notion.

Given a bounded scalar-valued function $f$ on a set $\Omega$ and an ultrafilter $\xi$ on $\Omega$, the {\em limit of $f$ along the ultrafilter $\xi$} is a uniquely defined number, $y$, with the property that for each $\e>0$,
\begin{equation}
\label{eq:ultralimit}
\{x\in\Omega\colon \abs{f(x)-y}<\e\}\in\xi.
\end{equation}
The limit along an ultrafilter, or an {\em ultralimit}, for short, is denoted $\lim_{x\to\xi}f(x)$.
Unlike the usual limit, the ultralimit of a bounded function along a fixed ultrafilter always exists, the proof of which fact mimicks the classical Heine--Borel compactness argument for the closed interval.
This observation makes the ultralimit a very powerful tool. Its downside is a highly non-constructive nature: typically, the value of an ultralimit of a particular function cannot be computed explicitely except in the ``uninteresting'' situations where it coincides with the usual limit.

The correspondence $\xi\mapsto \lim_{x\to\xi}f(x)$ defines a continuous function $\bar f$ on $\beta\Omega$, which is a unique continuous extension of $f$ over the Stone-\v Cech compactification $\beta\Omega$. Here, as is usual in set-theoretic topology and analysis, we identify every point $x$ of $\Omega$ with the corresponding principal (trivial) ultrafilter, $\bar x$, consisting of all subsets of $\Omega$ which contain $x$ as an element.

If an ultrafilter $\xi$ is fixed, then the correspondence $f\mapsto f(\xi)$ is a linear multiplicative functional of norm one on $\ell^\infty(\Omega)$, sending the function $1$ to $1$. It turns out that every linear multiplicative functional $\phi$ of norm one on $\ell^\infty(\Omega)$ sending $1$ to $1$ is of this form, that is, is the ultralimit along some ultrafilter on $\Omega$. This is, in fact, a rather simple observation: suffices to restrict $\phi$ to the set of all $\{0,1\}$-valued functions on $\Omega$ and notice that the image of every such function is necessarily either $0$ or $1$; the family $\xi$ of all sets $A\subseteq\Omega$ with $\phi(\chi_A)=1$ is now seen to be an ultrafilter, and an approximation argument with finite linear combinations shows that for every $f\in\ell^\infty(\Omega)$ one must have $\phi(f)=\lim_{x\to\xi}f(x)$.
In this way  the maximal ideal space of $\ell^\infty(\Omega)$ is identified with the space of ultrafilters $\beta\Omega$, that is, the Stone-\v Cech compactification of $\Omega$. Thus, the $C^\ast$-algebras $\ell^\infty(\Omega)$ and $C(\beta\Omega)$ are isomorphic. An isomorphism is given by the map $f\mapsto\bar f$, where $\bar f$ is the unique continuous extension of $f$ over $\beta\Omega$ mentioned above.


Given a $C^\ast$-algebra, an {\em ideal} $I$ of $A$ is a closed linear subspace stable under multiplication by elements of $A$. The quotient algebra $A/I$ is again a $C^\ast$-algebra (which is in general not an easy fact to prove). If $A$ is a commutative unital $C^\ast$-algebra and $I$ is a non-trivial ideal ($I\neq A$), then $A/I$ is isomorphic to an algebra of continuous functions on a suitable closed subspace $Y$ of the maximal ideal space $X$ of $A$. A functional $x\in X$ belongs to $Y$ if and only if it factors through the quotient map $\pi\colon A\to A/I$, that is, the kernel of $x\colon A\to \C$ contains $I$.

Conversely, every compact subspace of $X$ determines an ideal of $C(X)$. 

A link with the Boolean algebra setting is provided by the following observation: every ideal $I$ of subsets of $\Omega$ generates an ideal $\tilde I$ of the $C^\ast$-algebra $\ell^{\infty}(\Omega)$, as the smallest ideal of $A$ containing characteristic functions of all elements of $I$. Now one can verify without difficulty that the maximal ideal space of the $C^\ast$-algebra $\ell^{\infty}(\Omega)/\tilde I$ is the Stone space of the Boolean algebra $2^{\Omega}/I$. In fact, every ideal of $\ell^{\infty}(\Omega)$ is of this form.

\begin{definition}
Let $A$ be a commutative unital $C^\ast$-algebra, ${\mathscr F}$ a subset of $A$, and $I$ an ideal of $A$. For every $\e>0$, define the {\em $\e$-fat shattering dimension of $\mathscr F$ modulo $I$,} denoted $\fat_{\e}({\mathscr F}\modd I)$, as the $\e$-fat shattering dimension of $\mathscr F$ viewed as a function class on the maximal ideal space $Y$ of $A/I$. 

In a more detailed way, we denote $\pi\colon A\to A/I$ the quotient homomorphism.
A finite set $B\subseteq Y$ is $\e$-fat shattered by $\mathscr F$ if for some function $h\colon B\to [0,1]$ and every $C\subseteq B$ there is $f_C\in {\mathscr F}$ with
\[\begin{cases}
y(\pi(f_C))>h(y)+\e,&y\in C, \\
y(\pi(f_C))<h(y)-\e,&y\notin C.
\end{cases}
\]
Here elements $y\in Y$ are treated as functionals on $A/I$. The $\e$-fat shattering dimension of $\mathscr F$ modulo $I$, denoted $\fat_{\e}({\mathscr F}\modd I)$ is the supremum of cardinalities of finite subsets of the maximal ideal space of $A/I$ $\e$-fat shattered by $\mathscr F$.
\end{definition}

\begin{definition}
\label{d:fatmodc}
Let $\mathscr F$ be a function class on a domain $\Omega$, and let $\e>0$. We call the {\em $\e$-fat shattering dimension of $\mathscr F$ modulo countable sets} the value $\fat_{\e}({\mathscr F}\modd \tilde I)$, where $\tilde I$ is a $C^\ast$-algebra ideal of $\ell^{\infty}(\Omega)$ generated by characteristic functions of countable sets.
\end{definition}

Now we reformulate Definition \ref{d:fatmodc} avoiding the $C^\ast$-algebraic terminology. Let $\beta_{\omega_1}\Omega$ denote the collection of all points of $\beta\Omega$ which, viewed as ultrafilters on $\Omega$, only contain uncountable sets. The $\e$-fat shattering dimension of $\mathscr F$ modulo countable sets is the usual $\e$-fat shattering dimension of the class of functions $f\in\mathscr F$ extended over $\beta\Omega$ by continuity and then restricted to $\beta_{\omega_1}\Omega$.

We have an analogue of Theorem \ref{th:shatter}.

\begin{theorem}
\label{th:fshatter}
Let $\mathscr F$ be a class of measurable functions on a standard Borel domain $\Omega$, and let $I$ be an ideal of the $C^\ast$-algebra $\ell^{\infty}(\Omega)$. 
Fix any $\e>0$.
The following are equivalent.
\begin{enumerate}
\item 
\label{i:fvc}
The $\e$-fat shattering dimension of $\mathscr F$ modulo $I$ is at least $n$.
\item 
\label{i:fshattered}
There exists a family $A_1,A_2,\ldots,A_n$ of measurable subsets of $\Omega$ whose indicator functions do not belong to $I$, which is $\e$-fat shattered by $\mathscr F$ in the following sense: there is a witness function $h\colon \{1,2,\ldots,n\}\to [0,1]$ and for each $J\subseteq \{1,2,\ldots,n\}$ there is a $f_J\in {\mathscr F}$ such that
\begin{equation}
\label{eq:fatshatt}
\begin{array}{c}
(i\in J\wedge x\in A_i) \Rightarrow {f_J}(x)>h(i)+\e, \\
(i\notin J\wedge x\in A_i) \Rightarrow {f_J}(x)<h(i)-\e.
\end{array}
\end{equation}
\end{enumerate}
\end{theorem}

\begin{proof}
Before proceeding to the argument, let us remind that ultrafilters on $\Omega$ are viewed sometimes as mere points of the Stone-\v Cech compactification $\beta\Omega$, and sometimes as families of subsets of $\Omega$. Every point $x\in\Omega$ is canonically {\em identified} with the corresponding principal ultrafilter $\bar x$, and every bounded function $f$ on $\Omega$ admits a canonical continuous extension over $\beta\Omega$ via the rule $\bar f(\xi)=\lim_{x\to\xi}f(x)$. Notice that this definition implies $\bar f(\bar x)=f(x)$ whenever $x\in\Omega$.

(\ref{i:vc})$\Rightarrow$(\ref{i:shattered}). Let $Y\subseteq \beta\Omega$ denote the maximal ideal space of the $C^\ast$-algebra $\ell^{\infty}(\Omega)/I$. In other words, $\ell^{\infty}(\Omega)/I\cong C(Y)$.
There exist $n$ elements of $Y$ which are $\e$-fat shattered by $\mathscr F$, let us say $\xi_1,\ldots,\xi_n$. Recall that these are ultrafilters on $\Omega$, that is, families of subsets of the domain. Choose a witness function $h\colon \{1,2,\ldots,n\}\to [0,1]$, and select for every $J\subseteq \{1,2,\ldots,n\}$ a function $f_J\in {\mathscr F}$ whose ultralimit along $\xi_i$ is $>h(i)+\e$ if $i\in J$, and is $<h(i)-\e$ otherwise. For all $i=1,2,\ldots,n$, denote by
\begin{equation}
\label{eq:fbigcap}
\widetilde{A_i} = \bigcap_{J\ni i}\left\{\xi\in\beta\Omega\colon \widetilde{f_J}(\xi)>h(i)+\e\right\}\cap \bigcap_{J\not\ni i} \left\{\xi\in\beta\Omega\colon \overline{f_J}(\xi)<h(i)-\e\right\},
\end{equation}
and consider $A_i = \widetilde{A_i}\cap\Omega$. For every $i$ one has $\xi_i\in \widetilde{A_i}$ by the choice of the functions $f_J$. Since the value $\overline{f_J}(\xi_i)$ is the ultralimit of $f_J$ along $\xi_i$, it follows from the definition of an ultralimit (\ref{eq:ultralimit}) that each of the $2^n$ sets appearing in Eq. (\ref{eq:fbigcap}) belongs to $\xi_i$, and since $\xi_i$ is closed under finite intersections, one has $A_i\in\xi_i$. Equivalently, $\overline{\chi_{A_i}}(\xi_i)=1$, which implies that $\chi_{A_i}\notin I$ (as every function in the ideal $I$ --- or, a bit more precisely, its unique continuous extension over $\beta\Omega$ --- identically vanishes on $Y$). Since the functions $f_J$ are measurable with regard to the Borel structure on $\Omega$, so are the sets $A_i$. The condition (\ref{i:shattered}) is verified by the definition of the sets $A_i$. 
\par
(\ref{i:shattered})$\Rightarrow$(\ref{i:vc}). Let $A_1,A_2,\ldots,A_n$ be a family of subsets of $\Omega$ satisfying (\ref{i:shattered}). Their topological closures $\overline{A_i}$ taken in $\beta\Omega$ satisfy  
\[(i\in J\wedge \xi\in \overline{A_i}) \Rightarrow \overline{f_J}(\xi)>h(i)+\e,\]
\[(i\notin J\wedge \xi\in \overline{A_i}) \Rightarrow \overline{f_J}(\xi)<h(i)-\e.\]
The condition $\chi_{A_i}\notin I$ can be reformulated as $\overline{A_i}\cap Y\neq\emptyset$. Choose $\xi_i\in \overline{A_i}\cap Y$ for every $i=1,2,\ldots,n$. The set $\{\xi_i\}_{1=1}^n$ is $\e$-fat shattered by the functions $\bar f$, $f\in {\mathscr F}$ with the witness function $\xi_i\mapsto h(i)$. 
\end{proof}

\begin{remark}
Note that we have not used the assumption of measurability of subsets $A_i$ in the proof of the implication (\ref{i:shattered})$\Rightarrow$(\ref{i:vc}).
\end{remark}

\begin{corollary}
Let $\mathscr F$ be a class of $[0,1]$-valued functions on $\Omega$ and let $\e>0$. The $\e$-fat shattering dimension of $\mathscr F$ equals the $\e$-fat shattering dimension of the set of functions $\bar f$, $f\in {\mathscr F}$ on $\beta\Omega$. \qed
\end{corollary}

\begin{corollary}
Let $\mathscr F$ be a class of $[0,1]$-valued functions on $\Omega$ and let $\e>0$. The $\e$-fat shattering dimension of $\mathscr F$ modulo countable sets is the supremum of cardinalities of finite families $A_1,A_2,\ldots,A_n$ of uncountable subsets of $\Omega$ which are $\e$-fat shattered by $\mathscr F$ in the sense of Condition (\ref{eq:fatshatt}) with a suitable witness function $h\colon \{1,2,\ldots,n\}\to [0,1]$. \qed
\end{corollary}

\section{\label{s:necessary}Finiteness of combinatorial dimension modulo countable sets as a necessary condition}

In this Section, we remark that, similarly to the classical case of distribution-free learning, 
finiteness of VC dimension modulo countable sets is necessary for PAC learnability of a concept class under non-atomic measures, but this is not the case for fat shattering dimension of a function class.

\begin{lemma}
\label{l:supports}
Every uncountable Borel subset of a standard Borel space supports a non-atomic Borel probability measure.
\end{lemma}

\begin{proof}
Let $A$ be an uncountable Borel subset of a standard Borel space $\Omega$, that is, $\Omega$ is a Polish space equipped with its Borel structure. According to Souslin's theorem (see e.g. Theorem 3.2.1 in \cite{arveson}), there exists a Polish (complete separable metric) space $X$ and a continuous one-to-one mapping $f\colon X\to A$. The Polish space $X$ must be therefore uncountable, and so supports a non-atomic probability measure, $\nu$. The direct image measure $f_{\ast}\nu =\nu(f^{-1}(B))$ on $\Omega$ is a Borel probability measure supported on $A$, and it is non-atomic because the inverse image of every singleton is a singleton in $X$ and thus has measure zero.
\end{proof}

The following result makes no measurability assumptions on the concept class.

\begin{theorem}
\label{th:necessary}
Let $\mathscr C$ be a concept class on a domain $(\Omega,{\mathscr B})$ which is a standard Borel space. If $\mathscr C$ is PAC learnable under non-atomic measures, then the VC dimension of $\mathscr C$ modulo countable sets is finite.
\end{theorem}

\begin{proof}
This is just a minor variation of a classical result for distribution-free PAC learnability (Theorem 2.1(i) in \cite {BEHW}; we will follow the proof as presented in \cite{vidyasagar2003}, Lemma 7.2 on p. 279). 

Suppose $\VC({\mathscr C}\modd\omega_1)\geq d$. According to Theorem \ref{th:shatter}, there is a family of uncountable Borel sets $A_i$, $i=1,2,\dots, d$, shattered by $\mathscr C$ in our sense. Using Lemma \ref{l:supports}, select for every $i=1,2,\ldots,d$ a non-atomic probability measure $\mu_i$ supported on $A_i$, and let $\mu = \frac 1d\sum_{i=1}^d\mu_i$. This $\mu$ is a non-atomic Borel probability measure, giving each $A_i$ equal weight $1/d$. See Figure \ref{fig:shattered-4}.

\begin{figure}[ht]
\begin{center}
  \scalebox{0.275}[0.275]{\includegraphics{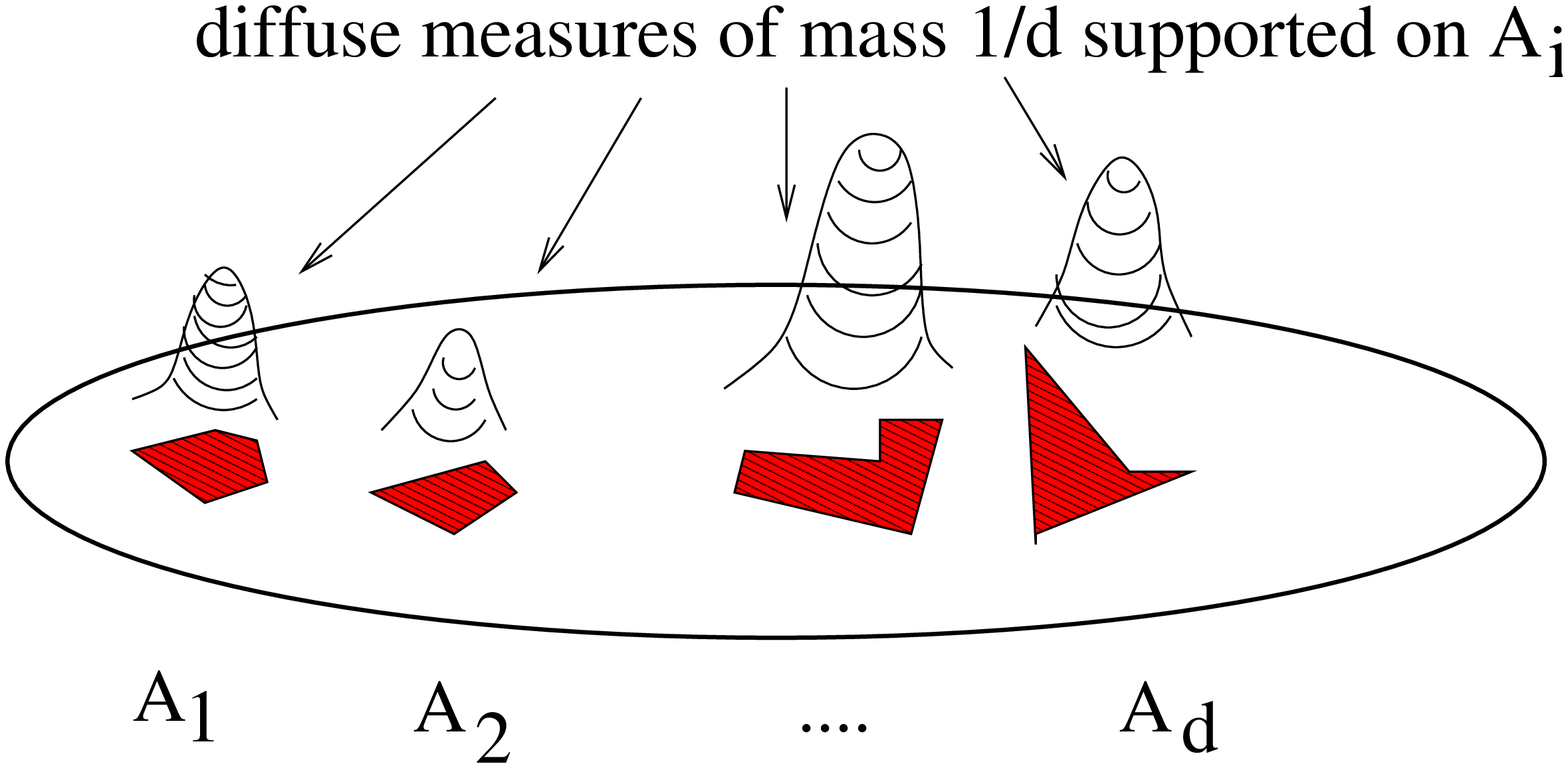}} 
\caption{Construction of the measure $\mu$.}
\label{fig:shattered-4}
\end{center}
\end{figure}

For every $d$-bit string $\sigma$ there is a concept $C_\sigma\in{\mathscr C}$ which contains all $A_{i}$ with $\sigma_i=1$ and is disjoint from $A_i$ with $\sigma_i=0$. 
If $A$ and $B$ take constant values on all the sets $A_i$, $i=1,2,\ldots,d$, then $d_{\mu}(A,B)$ is just the normalized Hamming distance between the corresponding $d$-bit strings. Now, given $A\in{\mathscr C}$ and $0\leq k\leq d$, there are 
\[\sum_{k\leq 2\e d}{d \choose k}\]
concepts $B$ with $d_{\mu}(A,B)\leq 2\e$. This allows to get the following lower bound on the number of pairwise $2\e$-separated concepts:
\[\frac {2^d}{\sum_{k\leq 2\e d}{d\choose k}}.\]
The Chernoff--Okamoto bound allows to estimate the above expression from below by $\exp[2(0.5-2\e)^2d]$. We conclude: the metric entropy of $\mathscr C$ with regard to $\mu$ is bounded from below by
\[M(2\e,{\mathscr C},\mu)\geq \exp[2(0.5-2\e)^2d].\]

The assumption $\VC({\mathscr C}\modd\omega_1)=\infty$ now implies that for every $0<\e<0.25$,
\[\sup_{P\in {\mathcal P}}M(2\e,{\mathscr C},\mu)=\infty,\]
where $\mathcal P$ denotes the family of all non-atomic measures on $\Omega$.
By Lemma 7.1 in \cite{vidyasagar2003}, p. 278, the class $\mathscr C$ is not PAC learnable under $\mathcal P$.
\end{proof}

On the contrary, a function class $\mathscr F$ can be PAC learnable under non-atomic measures and still have an infinite fat-shattering dimension modulo countable sets. The following is an adaptation of Example 2.10 in \cite{pestov:2010}.

\begin{example}
\label{ex:fprecpacinfdim}
For a given $n\in\N$, call any interval of the form $[i/n,(i+1)/n]$, $i=0,1,\ldots,n-1$ an {\em interval of order $n$}. Form the class $\mathscr C_n$ consisting of all unions of less than $\sqrt[3]{n}$ intervals of order $n$. Let $\mathscr C$ be the union of classes ${\mathscr C}_n$, $n\in\N$. Now we will transform $\mathscr C$ into a function class. With this purpose, establish a bijection $i$ between ${\mathscr C}$ and the rational points of the interval $[0,1/3]$. Let $\mathscr F$ consist of all functions of the form $f_C$, where
\[f_C(x) = \chi_C(x)+(-1)^{\chi_C(x)}i(C).
\]
Each function $f_C$ takes its (rational) values in $[0,1/3]\cup [2/3,1]$ and is uniquely identifiable by its value at {\em any} single point $x\in [0,1]$. For this reason, the class $\mathscr F$ is (exactly) learnable. A learning rule is given, for instance, by ${\mathcal L}(x,r)=i^{-1}(\min\{r,1-r\})$, where $(x,r)$ is a learning $1$-sample.

At the same time, ${\mathrm{fat}}_{1/6}({\mathscr F}\modd \omega_1)=\infty$. Indeed, given any $k\in\N$, an arbitrary collection $I_1,I_2,\ldots,I_k$ of $k$ pairwise distinct intervals of order $n=k^3$ is $1/6$-shattered by the functions $f_C$, $C\in {\mathscr C}_n$ with the witness function taking a constant value $1/2$. 

This example can be further modified. For instance, one can consider a larger  class $\tilde{\mathscr F}$ consisting of all functions $f$ for which there exists a $g\in {\mathscr F}$ with $\{x\colon f(x)\neq g(x)\}$ being a universal null set. The class $\tilde{\mathscr F}$ is probably exactly learnable by the same learning rule $\mathcal L$ as above.
\end{example}

\section{The universally separable case}

In this Section we will express our versions of the combinatorial dimension modulo countable sets in terms of the corresponding classical notions. Namely, we will prove that $\VC({\mathscr C}\modd\omega_1)\leq d$ if and only if every countable subclass of $\mathscr C$ has VC dimension $d$ outside of a suitable countable set, and similarly for fat shattering dimension. 

\begin{lemma}
Let $\mathscr F$ be a universally separable function class, with a universally dense countable subset $\mathscr F^\prime$. Then for every $\e>0$
\[{\mathrm{fat}}_{\e}({\mathscr F}) = \fat_{\e}({\mathscr F}^\prime).\]
\label{l:shatt}
\end{lemma}

\begin{proof}
For every $f\in {\mathscr F}$ there is a sequence $(f_n)$ of elements of $\mathscr F^{\prime}$ which converges to $f$ pointwise: given a finite $A\subseteq\Omega$ and an $\gamma>0$, there is an $N$ such that whenever $n\geq N$, one has $\abs{f(x)-f_n(x)}<\gamma$ for all $x\in A$. This means that if $A$ is $\e$-fat shattered by $\mathscr F$, it is equally well shattered by $\mathscr F^{\prime}$, with the same witness function. This observation establishes the inequaity $\fat_{\e}({\mathscr F}) \leq \fat_{\e}({\mathscr F}^\prime)$, while the converse inequality is trivially true.
\end{proof}

Since for a concept class $\mathscr C$ one has $\VC({\mathscr C})=\fat_{\e}({\mathscr C})$ whenever $\e<1/2$, we obtain:

\begin{corollary}
Let $\mathscr C$ be a universally separable concept class, and let $\mathscr C^\prime$ be a universally dense countable subset of $\mathscr C$. Then
\[\VC({\mathscr C}) = \VC({\mathscr C}^\prime).\]
\label{l:cshatt}
\end{corollary}

While a version of the following result for fat shattering dimension covers the VC dimension as a particular case, the proof is technically more complicated, and we feel that the complications obscure the simple idea of the proof for VC dimension. For this reason, we give a separate presentation for VC dimension first.

\begin{theorem}
\label{th:uscc}
For a universally separable concept class $\mathscr C$, the following conditions are equivalent.
\begin{enumerate}
\item\label{c:uscc1} 
$\VC({\mathscr C}\modd\omega_1)\leq d$.
\item \label{c:uscc2} 
There exists a countable subset $A\subseteq \Omega$ such that $\VC({\mathscr C}\upharpoonright(\Omega\setminus A))\leq d$.
\end{enumerate}
\end{theorem}

\begin{proof}
(\ref{c:uscc1})$\Rightarrow$(\ref{c:uscc2}): Choose a countable universally dense subfamily $\mathscr C^\prime$ of $\mathscr C$. Let $\mathscr B$ be the smallest Boolean algebra of subsets of $\Omega$ containing $\mathscr C^\prime$. Denote by $A$ the union of all elements of $\mathscr B$ that are countable sets. Clearly, $\mathscr B$ is countable, and so $A$ is a countable set. 

Let a finite set $B\subseteq\Omega\setminus A$ be shattered by $\mathscr C$. Then, by Corollary \ref{l:cshatt}, it is shattered by ${\mathscr C}^\prime$. Select a family $\mathscr S$ of $2^{\abs B}$ sets in ${\mathscr C}^\prime$ shattering $B$.
For every $b\in B$ the set 
\[[b]=\bigcap_{b\in C\in {\mathscr S}} C \cap \bigcap_{b\notin C\in {\mathscr S}}C^c\] 
is uncountable (for it belongs to $\mathscr B$ yet is not contained in $A$), and the collection of sets $[b]$, $b\in B$ is shattered by ${\mathscr C}^\prime$. According to (\ref{c:uscc1}), $\abs B\leq d$, from which we deduce (\ref{c:uscc2}). 
Notice that this establishes the inequality $\VC({\mathscr C}\upharpoonright (\Omega\setminus A))\leq\VC({\mathscr C}\modd{\omega_1})$. 

(\ref{c:uscc2})$\Rightarrow$(\ref{c:uscc1}): Fix an $A\subseteq\Omega$ such that $\VC({\mathscr C}\modd A^c)\leq d$. 
Suppose a collection of $n$ uncountable sets $A_i$, $i=1,2,\ldots,n$ is shattered by $\mathscr C$ in our sense. The sets $A_i\setminus A$ are non-empty; pick a representative $a_i\in A_i\setminus A$, $i=1,2,\ldots,n$. The resulting set $\{a_i\}_{i=1}^n$ is shattered by $\mathscr C$, meaning $n\leq d$.
\end{proof}

Now a version for fat shattering dimension.

\begin{theorem}
\label{th:fuscc}
For a universally separable function class $\mathscr F$ and $\e>0$, the following conditions are equivalent.
\begin{enumerate}
\item\label{c:fuscc1} 
$\fat_{\e}({\mathscr F}\modd\omega_1)\leq d$.
\item \label{c:fuscc2} 
There exists a countable subset $A\subseteq \Omega$ such that $\fat_{\e}({\mathscr F}\upharpoonright(\Omega\setminus A))\leq d$.
\end{enumerate}
For a universally separable function class $\mathscr F$ and $\e>0$, the  conditions are equivalent.
\end{theorem}

\begin{proof}
(\ref{c:fuscc1})$\Rightarrow$(\ref{c:fuscc2}): For a function $f$ on $\Omega$ and $r\in\R$, denote 
\[[f<r] = \{x\in\Omega\colon f(x)<r\}\mbox{ and }
[f>r]=\{x\in\Omega\colon f(x)>r\}.\]
Let $\mathscr F^{\prime}$ be a countable universally dense subfamily of $\mathscr F$. Denote by $\mathscr B$ the smallest algebra of subsets of $\Omega$ containing all sets $[f<r]$, $[f>r]$ for $f\in \mathscr F^{\prime}$ and $r\in\Q$. Now denote by $A$ the union of all elements of $\mathscr B$ that are countable sets. Since $\mathscr B$ is countable, so is $A$. 

Let a finite set $B\subseteq\Omega\setminus A$ be $\e$-fat shattered by $\mathscr F$. Then, by Lemma \ref{l:shatt}, it is shattered by ${\mathscr F}^\prime$, and by Lemma \ref{l:rational}, there is a rational $\e^{\prime}>\e$ and a rational-valued function $h\colon B\to\Q$ such that $B$ is $\e^{\prime}$-fat shattered by a family $\mathscr S$ of $2^{\abs B}$ functions in ${\mathscr F}^\prime$ with $h$ as a witness function. 

For every $b\in B$ form the set 
\[
\begin{array}{rl}
[b]=\left\{x\in\Omega\colon \forall C\subseteq B,\right.&b\in C\Rightarrow
f_C(x)>h(b)+\e^{\prime}~\wedge \\
&\left. b\notin C\Rightarrow f_C(x)<h(b)-\e^{\prime}\right\}.
\end{array}
\]
The set $[b]$ belongs to the algebra of sets $\mathscr B$ and is not contained in $A$ (for instance, $b\in [b]$ and $b\notin A$). Therefore, $[b]$ is uncountable. If $b,c\in B$ and $b\neq c$, then $[b]\cap [c]=\emptyset$. Finally, the collection of sets $[b]$, $b\in B$ is $\e^{\prime}$-fat shattered by $\mathscr F^{\prime}$ with $h$ as a witness function, hence $\e$-fat shattered. Since $\abs B\leq d$, we have proved (\ref{c:fuscc2}), and established the inequality $\fat_{\e}({\mathscr F}\upharpoonright (\Omega\setminus A))\leq\fat_{\e}({\mathscr F}\modd{\omega_1})$. 

(\ref{c:uscc2})$\Rightarrow$(\ref{c:uscc1}): Fix a countable subset $A\subseteq\Omega$ such that $\fat_{\e}({\mathscr F}\modd A^c)\leq d$. 
Suppose a collection of $n$ uncountable sets $A_i$, $i=1,2,\ldots,n$ is $\e$-fat shattered by the function class $\mathscr F$. The sets $A_i\setminus A$ are non-empty, so we can select a representative $a_i$ in each one of them, $i=1,2,\ldots,n$. The resulting set $\{a_i\}_{i=1}^n$ is $\e$-fat shattered by $\mathscr F$, meaning $n\leq d$.
\end{proof}

\begin{corollary}
\label{c:usl}
Let $\mathscr C$ be a universally separable concept class on a Borel domain $\Omega$. If $d=\VC({\mathscr C}\modd {\omega_1})<\infty$, then $\mathscr C$ is a uniform Glivenko-Cantelli class with respect to non-atomic measures and consistently PAC learnable under non-atomic measures, with a standard sample complexity corresponding to $d$.
\end{corollary}

\begin{proof}
The class $\mathscr C$ has finite VC dimension in the complement to a suitable countable subset $A$ of $\Omega$, hence $\mathscr C$ is a universal Glivenko-Cantelli class (in the classical sense) in the standard Borel space $\Omega\setminus A$. But $A$ is a universal null set in $\Omega$, hence clearly $\mathscr C$ is universal Glivenko-Cantelli with respect to non-atomic measures.

The class $\mathscr C$ is distribution-free consistently PAC learnable in the domain $\Omega\setminus A$, with the standard sample complexity $s(\e,\delta,d)$. Let $\mathcal L$ be any consistent learning rule for $\mathscr C$ in $\Omega$. The restriction of $\mathcal L$ to $\Omega\setminus A$ (more exactly, to $\cup_{n=1}^\infty\left((\Omega\setminus A)^n\times\{0,1\}^n \right)$) is a consistent learning rule for $\mathscr C$ restricted to the standard Borel space $\Omega\setminus A$, and together with the fact that $A$ has measure zero with respect to any non-atomic measure, it implies that $\mathcal L$ is a PAC learning rule for $\mathscr C$ under non-atomic measures, with the same sample complexity function $s(\e,\delta,d)$.
\end{proof}

Similarly, we obtain:

\begin{corollary}
\label{c:fusl}
Let $\mathscr F$ be a universally separable function class on a Borel domain $\Omega$. If for every $\e>0$ one has $d=\fat_{\e}({\mathscr F}\modd {\omega_1})<\infty$, then $\mathscr F$ is a uniform Glivenko-Cantelli class with respect to non-atomic measures and consistently PAC learnable under non-atomic measures, with a standard sample complexity corresponding to $d$.
\qed
\end{corollary}

Here are the two main conclusions of this Section. Notice that the following criteria {\em no longer} assume universal separability of the classes involved.

\begin{corollary}
\label{c:criterion}
For a concept class $\mathscr C$, the following are equivalent.
\begin{enumerate}
\item
\label{c:criterion1} $\VC$-dimension of $\mathscr C$ modulo countable sets is $\leq d$;
\item
\label{c:criterion2} For every countable subclass $\mathscr C^{\prime}$ of $\mathscr C$, there exists a countable $A\subseteq\Omega$ such that the $VC$-dimension of $\mathscr C^{\prime}$ restricted to $\Omega\setminus A$ is $\leq d$.
\end{enumerate}
\end{corollary}

\begin{proof}
(\ref{c:criterion1})$\Rightarrow$(\ref{c:criterion2}): the VC dimension modulo countable sets is monotone with respect to subclasses, so $\VC({\mathscr C}^{\prime}\modd\omega_1)\leq d$. Now Theorem \ref{th:uscc} gives the desired conclusion.

(\ref{c:criterion2})$\Rightarrow$(\ref{c:criterion1}): assume uncountable sets $A_1,A_2,\ldots,A_n$ are shattered by $\mathscr C$. Select a family $\mathcal S$ of $2^n$ concept classes that does the shattering. There is a countable $A$ such that $\VC({\mathcal S}\upharpoonright \Omega\setminus A)\leq d$. Choose a representative $a_i$ in each of the non-empty sets $A_i\setminus A$. Since the set $\{a_i\}_{i=1}^n$ is shattered by the family $\mathcal S$ restricted to $\Omega\setminus A$, one concludes that $n\leq d$.
\end{proof}

Similarly, one obtains:

\begin{corollary}
\label{c:fcriterion}
For a function class $\mathscr F$ and $\e>0$, the following are equivalent.
\begin{enumerate}
\item
\label{c:fcriterion1} $\fat_{\e}({\mathscr F}\modd\omega_1)\leq d$;
\item
\label{c:fcriterion2} For every countable subclass $\mathscr F^{\prime}$ of $\mathscr F$, one has $\fat_{\e}({\mathscr F}^{\prime}\upharpoonright \Omega\setminus A)\leq d$ for a suitable countable $A$ (which depends on $\mathscr F^{\prime}$). \qed
\end{enumerate}
\end{corollary}

\section{Proofs of two theorems from the Introduction}

Now we are in a position to prove the two main theorems \ref{th:main} and \ref{th:fmain}, just by putting together various results established in the article.

\subsection{Key to the proof of Theorem \ref{th:main}}

\noindent
(\ref{mainth:1})$\Rightarrow$(\ref{mainth:3}): this is Theorem \ref{th:necessary}.

\noindent
(\ref{mainth:3})$\Rightarrow$(\ref{mainth:4}): Corollary \ref{c:criterion}.

\noindent 
(\ref{mainth:4})$\Rightarrow$(\ref{mainth:4a}): assume that for every $d$ there is a countable subclass ${\mathscr C}_d$ of $\mathscr C$ with the property that the VC dimension of ${\mathscr C}_d$ is $\geq d$ after removing any countable subset of $\Omega$. Clearly, the countable class $\cup_{d=1}^\infty {\mathscr C}_d$ will have infinite VC dimension outside of every countable subset of $\Omega$, a contradiction.

\noindent
(\ref{mainth:4a})$\Rightarrow$(\ref{mainth:5}): as a consequence of a classical result of Vapnik and Chervonenkis, every countable subclass $\mathscr C^\prime$ is universal Glivenko-Cantelli with respect to all probability measures supported outside of some countable subset of $\Omega$, and a standard bound for the sample complexity $s(\delta,\e)$ only depends on $d$, from which the statement follows.

\noindent
(\ref{mainth:5})$\Rightarrow$(\ref{mainth:5a}): trivial.

\noindent
(\ref{mainth:5a})$\Rightarrow$(\ref{mainth:1}): this is Theorem \ref{th:countablesubclassesugc}, and the only implication requiring Martin's Axiom.

In the universally separable case, the implications (\ref{mainth:3})$\iff$(7) are due to Theorem \ref{th:uscc}, (\ref{mainth:3})$\Rightarrow$(8) follows from Corollary \ref{c:usl}, (8)$\Rightarrow$(9) is standard, and (9)$\Rightarrow$(1) trivial.
\qed

\subsection{Key to the proof of Theorem \ref{th:fmain}}

\noindent
(\ref{mainth:3})$\Rightarrow$(\ref{mainth:4}): Corollary \ref{c:fcriterion}.

\noindent 
(\ref{mainth:4})$\Rightarrow$(\ref{mainth:4a}): Assume that for some $\e>0$ and every value $d\in\N$ there is a countable subclass ${\mathscr F}_d$ of $\mathscr F$ with the property that the $\e$-fat shattering dimension of ${\mathscr F}_d$ is $\geq d$ after removing any countable subset of $\Omega$. Then the countable function class $\cup_{d=1}^\infty {\mathscr F}_d$ will have infinite $\e$-fat shattering dimension outside of every countable subset of $\Omega$, which is a contradiction.

\noindent
(\ref{mainth:4a})$\Rightarrow$(\ref{mainth:5}): Combining the assumption with Theorem 2.5 in \cite{ABDCBH}, one concludes that every countable subclass $\mathscr F^\prime$ of $\mathscr F$ is universal Glivenko-Cantelli with respect to all probability measures supported outside of a suitable countable subset of $\Omega$, with a standard bound for the sample complexity $s(\delta,\e)$ only depending on $d(\e)$.

\noindent
(\ref{mainth:5})$\Rightarrow$(\ref{mainth:5a}): trivial.

\noindent
(\ref{mainth:5a})$\Rightarrow$(\ref{mainth:1}): Theorem \ref{th:countablesubclassesugc}. This is the the only implication requiring Martin's Axiom.

In the universally separable case, the equivalence of (\ref{mainth:1}) and (7) is the statement of Theorem \ref{th:fuscc}, (7)$\Rightarrow$(8) is Corollary \ref{c:fusl}, and (8)$\Rightarrow$(9) is standard.
\qed

Note again that the implication (\ref{fmainth:1})$\Rightarrow$(\ref{fmainth:3}) is in general invalid, cf. Example \ref{ex:fprecpacinfdim}.

\section{Conclusion and Open Problems}

We have characterized concept classes $\mathscr C$ that are distribution-free PAC learnable under the family of all non-atomic probability measures on the domain. The criterion is obtained without any measurability conditions on the concept class, but at the expense of making a set-theoretic assumption in the form of Martin's Axiom. In fact, assuming Martin's Axiom makes things easier, and as this axiom is very natural, perhaps it deserves its small corner within the foundations of statistical learning.

Generalizing the result over function classes, using a version of the fat shattering dimension modulo countable sets, did not pose particular technical difficulties. However the finiteness of this combinatorial parameter is no longer necessary for PAC learnability of a function class under non-atomic measures, just like it is the case for the classical distribution-free situation. 

It would be still interesting to know if the present results hold without Martin's Axiom, under the assumption that the concept class $\mathscr C$ is image admissible Souslin  (\cite{dudley}, pages 186--187). The difficulty here is selecting a measurable learning rule $\mathcal L$ with the property that the images of all learning samples $(\sigma,C\cap\sigma)$, $\sigma\in\Omega^n$, are uniform Glivenko-Cantelli. An obvious route to pursue is the recursion on the Borel rank of $\mathscr C$, but we were unable to follow it through.

Now, a concept class $\mathscr C$ will be learnable under non-atomic measures provided there is a hypothesis class $\mathscr H$ which has finite VC dimension and such that every $C\in{\mathscr C}$ differs from a suitable $H\in {\mathscr H}$ by a null set. If $\mathscr C$ consists of all finite and all cofinite subsets of $\Omega$, this $\mathscr H$ is given by $\{\emptyset,\Omega\}$. One may conjecture that $\mathscr C$ is learnable under non-atomic measures if and only if it admits such a ``core'' $\mathscr H$ having finite VC dimension. Is this  true?

Another natural question is: can one characterize concept classes that are uniformly Glivenko--Cantelli with respect to all non-atomic measures? Apparently, this task requires yet another version of shattering dimension, which is strictly intermediate between Talagrand's ``witness of irregularity'' \cite{talagrand96} and our VC dimension modulo countable sets. We do not have a viable candidate.

Is it possible to construct an example of a concept class of finite VC dimension which is not consistently PAC learnable \cite{DD,BEHW} without additional set-theoretical assumptions, just under the ZFC axiomatics?

Finally, our investigation open up a possibility of linking learnability and VC dimension to Boolean algebras and their Stone spaces. This could be a glib exercise in generalization for its own sake, or maybe something deeper if one manages to invoke model theory and forcing.

\subsection*{Acknowledgements}

The author is most grateful to two anonymous referees for their thorough reading of the paper and numerous useful suggestions which have helped to improve the presentation considerably. Of course the remaining imperfections are all author's own.

\bibliographystyle{splncs}

\end{document}